\documentclass[10pt,journal]{IEEEtran} 
\usepackage[ruled,linesnumbered]{algorithm2e}
\usepackage{amsthm}
\usepackage{amsmath}
\usepackage{slashbox}
\usepackage{graphicx}
\usepackage{hhline}
\usepackage{epsfig,graphics,psfrag,amsmath,amssymb}
\usepackage{cite}
\usepackage{float}
\usepackage{tocvsec2}
\usepackage{color}
\usepackage{verbatim}
\usepackage{mathtools}
\usepackage{hhline}
\usepackage{enumerate}
\usepackage{arydshln}
\usepackage{bbm}
\usepackage{stackengine}
\usepackage{mathrsfs}
\usepackage{enumitem}
\usepackage{caption}
\usepackage{stfloats}
\usepackage{amssymb}
\usepackage{booktabs}
\usepackage{makecell}

\usepackage{hhline}
\usepackage{graphicx}     
\usepackage{caption}      
\usepackage{subcaption}   

\usepackage[export]{adjustbox} 
\usepackage{hyperref}       
\usepackage{setspace}

\usepackage[symbol]{footmisc}

\newtheorem{thm}{Theorem}
\newtheorem{lem}{Lemma}

\newtheorem{rem}{Remark}

\newtheorem{assump}{Assumption}

\newtheorem{problem}{Problem}

\makeatletter

\newcommand{\Rmnum}[1]{\expandafter\@slowromancap\romannumeral #1@}
\makeatother

\makeatletter
\newcommand*\bigcdot{\mathpalette\bigcdot@{.5}}
\newcommand*\bigcdot@[2]{\mathbin{\vcenter{\hbox{\scalebox{#2}{$\m@th#1\bullet$}}}}}
\makeatother

\def\@#1{\pmb{#1}}
\def\b#1{\mathbb{#1}}

\def\s#1{\mathsf{#1}}

\def\ca#1{\mathcal{#1}}

\def\v#1{\boldsymbol{#1}}

\newcommand\aleq{\mathrel{\stackrel{\makebox[0pt]{\mbox{\normalfont\tiny (a)}}}{\leq}}}
\newcommand\bleq{\mathrel{\stackrel{\makebox[0pt]{\mbox{\normalfont\tiny (b)}}}{\leq}}}

\newcommand\aeq{\mathrel{\stackrel{\makebox[0pt]{\mbox{\normalfont\tiny (a)}}}{=}}}
\newcommand\beq{\mathrel{\stackrel{\makebox[0pt]{\mbox{\normalfont\tiny (b)}}}{=}}}

\DeclareMathOperator*{\argmin}{arg\,min}

\allowdisplaybreaks

\begin{document}
\title{{Decentralized Federated Learning With Energy Harvesting Devices}}
\author{Kai Zhang$^{\dagger}$, Xuanyu Cao$^{\ddagger}$, \emph{Senior Member}, \emph{IEEE}, and Khaled B. Letaief$^{\dagger}$, \emph{Fellow}, \emph{IEEE}\\
	$^{\dagger}$Department of Electronic and Computer Engineering, The Hong Kong University of Science and Technology
	\\$^{\ddagger}$School of Electrical Engineering and Computer Science, Washington State University
	\\Email: kzhangbn@connect.ust.hk, xuanyu.cao@wsu.edu, eekhaled@ust.hk
	}

\maketitle

\begin{abstract}

Decentralized federated learning (DFL) enables edge devices to collaboratively train models through local training and fully decentralized device-to-device (D2D) model exchanges.
However, these energy-intensive operations often rapidly deplete limited device batteries, reducing their operational lifetime and degrading the learning performance.
To address this limitation, we apply energy harvesting technique to DFL systems, allowing edge devices to extract ambient energy and operate sustainably.
We first derive the convergence bound for wireless DFL with energy harvesting, showing that the convergence is influenced by partial device participation and transmission packet drops, both of which further depend on the available energy supply.
To accelerate convergence, we formulate a joint device scheduling and power control problem and model it as a multi-agent Markov decision process (MDP).
Traditional MDP algorithms (e.g., value or policy iteration) require a centralized coordinator with access to all device states and exhibit exponential complexity in the number of devices, making them impractical for large-scale decentralized networks.
To overcome these challenges, we propose a fully decentralized policy iteration algorithm that leverages only local state information from two-hop neighboring devices, thereby substantially reducing both communication overhead and computational complexity.
We further provide a theoretical analysis showing that the proposed decentralized algorithm achieves asymptotic optimality.
Finally, comprehensive numerical experiments on real-world datasets are conducted to validate the theoretical results and corroborate the effectiveness of the proposed algorithm.

\end{abstract}


\section{Introduction}

Machine learning (ML) is a fundamental technology underpinning numerous artificial intelligence (AI) applications, including autonomous vehicles \cite{kiran2021deep} and smart cities \cite{mohammadi2018enabling}.
As data generated by edge devices (e.g., mobile phones, wearables, sensors) continues to proliferate, there is an increasing interest in applying ML at the network edge \cite{letaief2019roadmap,letaief2021edge}.
However, traditional centralized ML frameworks require aggregating raw data on a remote server, which results in substantial network traffic and introduces significant privacy and security concerns.
As a remedy, federated learning (FL) has recently been introduced, enabling decentralized model training across edge devices without centralizing raw data at a single site \cite{mcmahan2017communication}.
In FL, edge devices perform local training on private data and share only model parameters or gradients.
This architecture not only preserves data privacy but also reduces both bandwidth usage and energy consumption by mitigating the need for massive data transmissions \cite{shao2023survey,yang2020energy}.

Nevertheless, practical FL systems still suffer from substantial communication overhead.
The frequent parameter exchanges in FL often saturate limited network bandwidth and increase transmission delays, ultimately impairing learning efficiency.
To address these communication bottlenecks, extensive research has proposed diverse strategies aimed at improving FL’s communication efficiency.
Some studies focused on reducing communication rounds to limit unnecessary information exchange \cite{wang2019adaptive,george2020distributed}, while others sought to compress transmitted information by employing quantization \cite{bouzinis2023wireless,liu2022hierarchical} or sparsification \cite{lin2017deep} techniques.
A third line of work explicitly considered practical wireless communication aspects (e.g., noise, fading, interference) and optimized radio resource allocation to improve learning efficiency \cite{9210812,wan2021convergence}.

Most existing FL systems adopt a server-devices architecture, in which a central server communicates with all edge devices and coordinates the training process.
However, this architecture poses significant scalability and reliability issues, especially in large-scale deployments.
The central server can create network congestion and lead to a single point of failure.
Moreover, many practical multi-device networks lack a central entity capable of coordinating all devices.
For example, large-scale sensor networks may lack fusion centers, and sensors can only interact with other nearby neighbors.
Consequently, decentralized federated learning (DFL) has gained attention as an appealing alternative \cite{beltran2023decentralized,qu2021decentralized}.
In DFL, edge devices train local models and exchange model updates directly through device-to-device (D2D) links, eliminating the need for a central aggregator.
This approach mitigates single points of failure, relieves communication bottlenecks caused by a central aggregator, and reduces long-range communication costs.

Despite these advantages, DFL introduces new communication challenges in the absence of a central server orchestrating model exchanges. 
Without a coordinating entity, edge devices must self-schedule their transmissions, which can lead to severe interference, network congestion, and inefficient utilization of limited wireless resources.
To address these challenges, several strategies have been developed.
For instance, in \cite{xing2021federated}, the authors implemented DFL over both digital and analog wireless networks. They leveraged random linear coding to compress communications and utilized over-the-air computation to enable simultaneous analog transmissions.
In \cite{sun2022decentralized}, the authors proposed a momentum-based decentralized averaging algorithm, refining the local update process and using quantization technique to accelerate the convergence and reduce communication cost.
To further alleviate communication burdens, a sparsification strategy with adaptive coordination for wireless DFL was proposed in \cite{tang2022gossipfl}.
Moreover, periodic synchronization and compression were studied in \cite{liu2022decentralized} to strike an optimal balance between communication frequency and computational effort for DFL, while partial utilization of received messages was developed in \cite{ye2022decentralized} to ensure robust performance for wireless DFL with unreliable links.
Likewise, in \cite{du2023adaptive}, the authors developed a DFL algorithm for resource-constrained IoT networks that adaptively adjusts communication compression ratios, prunes inefficient links, and reallocates power to meet latency and energy constraints.

Despite extensive research, the prior works on DFL \cite{xing2021federated,sun2022decentralized,tang2022gossipfl,liu2022decentralized,ye2022decentralized,du2023adaptive} have either assumed finite energy supply at edge devices or entirely disregarded the energy constraints.
Nevertheless, relying exclusively on the limited energy stored in device batteries inevitably reduces the operational lifetime: once the battery is depleted, edge devices can no longer participate in the training process of DFL, thereby degrading the learning performance.
Moreover, frequently swapping or recharging batteries is impractical, especially when devices are situated in remote or inaccessible locations.
To overcome energy limitations and prolong device lifetimes, energy harvesting (EH) has emerged as a promising paradigm, enabling edge devices to collect ambient energy (e.g., solar, wind, or thermoelectric) and thereby operate sustainably.
Several recent studies have investigated FL with EH devices to address the energy limitations of battery-powered devices.
For instance, \cite{shen_federated_2022} investigated a device scheduling problem to accelerate the convergence of FL with EH devices.
Likewise, \cite{zeng_federated_2023} developed a joint client scheduling and resource allocation approach that reduces training latency while enhancing data utility for EH FL.
The authors of \cite{chen_energy_2022} further studied joint device selection and energy management for FL with EH devices, developing a bandwidth-efficient transmission scheme by exploiting the over-the-air computation.
Meanwhile, \cite{liu2021age} presented a joint device scheduling and bandwidth allocation algorithm aimed at minimizing the ages of local models and the global iteration latency. Lastly, \cite{an_online_2023} studied joint transceiver design and client selection, and proposed a Lyapunov-based online optimization algorithm to accelerate EH FL.
Notably, all these prior works on EH-enabled FL focus on centralized architectures in the presence of a central server.
It still remains unclear how EH can enhance the performance of DFL and how to design the optimal transmission policy to effectively utilize limited and uncertain harvested energy in fully decentralized scenarios.

In this paper, we are motivated to apply EH technique to DFL.
In contrast to centralized FL, EH-enabled DFL requires devices to self-coordinate local training and model exchanges over time-varying wireless channels and fluctuating energy supplies.
Such decentralized coordination makes device scheduling and power control considerably more complex, which can result in suboptimal utilization of scarce energy resources.
To overcome this challenge, we develop a multi-agent Markov decision process (MDP) framework to derive the optimal scheduling and transmission strategy for DFL with EH devices.
The proposed multi-agent MDP framework leverages the statistical knowledge of wireless channels and harvested energy to cope with their inherent stochastic characteristics, thus facilitating more efficient energy utilization and enhancing learning performance.
Although traditional MDP methods (e.g., value or policy iteration) can theoretically solve this problem, they require a centralized coordinator with access to all device states and exhibit exponential complexity in the number of devices, making them impractical for large-scale decentralized networks.
Thus, we propose a fully decentralized policy iteration algorithm that relies only on local state information, thereby substantially reducing communication overhead and computational burden.

%
%
\subsection{Contributions}

The main contributions of this paper are summarized as follows.

\begin{itemize}
	\item We propose an EH-enabled DFL framework for sustainable model training among energy-constrained edge devices.
	We derive the convergence bound for wireless DFL with EH devices, showing the impact of partial device participation and transmission packet drops, both of which are determined by the available energy supply.
	Building on the convergence bound, we formulate a joint device scheduling and transmission power control problem to accelerate the convergence.
	
	\item We develop a multi-agent MDP framework for the formulated problem. To tackle the prohibitive complexity and communication overhead of traditional centralized MDP methods, we propose a fully decentralized policy iteration algorithm that requires only local state information from two-hop neighboring devices. We further provide a theoretical analysis showing that the proposed decentralized algorithm achieves asymptotic optimality.
	
	\item  We validate the theoretical results and assess the performance of the proposed algorithm using real-world datasets. Experimental results show that the proposed algorithm consistently outperforms benchmarks, achieving superior test accuracy. We further investigate the effect of battery capacity on learning performance.
\end{itemize}

\subsection{Organization and Notations}

The remainder of this paper is structured as follows. In Section \ref{Sec_2}, we describe the system model. In Section \ref{Sec_3}, we present the convergence analysis and formulate the design problem. In Section \ref{Sec_MDP}, we introduce a multi-agent MDP framework and propose a decentralized policy iteration algorithm. Simulation results are presented in Section \ref{sim_result}, and Section \ref{Sec_conclusion} concludes the paper.

\emph{Notations:} The symbol $\b{R}$ represents the set of real numbers. Bold lowercase letters denote column vectors.
$\b{E}[\cdot]$ indicates the expectation operation. $\nabla$ denotes the gradient operator.
$[n]$ corresponds to the set $\left\lbrace 1,2,\ldots,n\right\rbrace $.
$|\cdot|$ and $\|\cdot\|$ denote the $\ell_1$ and $\ell_2$ norms, respectively.
$\b{P}(\cdot|\cdot)$ denotes the conditional probability. $\mathbbm{1}_{\left\lbrace \cdot \right\rbrace }$ is the indicator function.
The total variation distance between two probability measures $\mu$ and $\nu$ is defined as
$\textup{TV}(\mu,\nu) = \sup_{A} \big|\mu(A) - \nu(A)\big|$.

\section{System Model}\label{Sec_2}

\subsection{Decentralized Federated Learning Model}

\begin{figure}[tbp]
	\renewcommand\figurename{\small Fig.}
	\centering \vspace*{6pt} \setlength{\baselineskip}{10pt}
	\includegraphics[width = 0.49\textwidth,center,trim=1 0 20 11,clip]{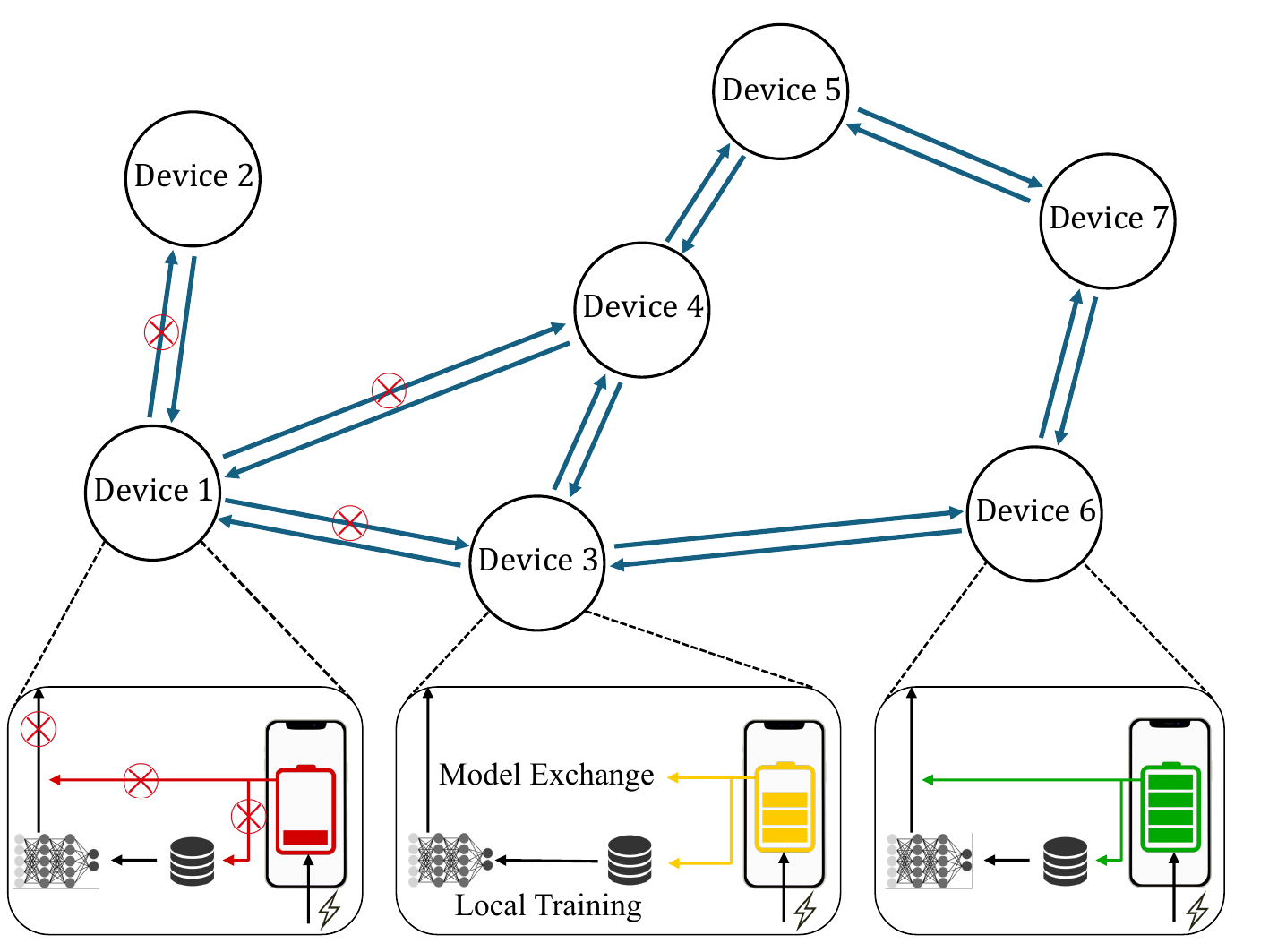}
	\caption{An illustration of the wireless DFL system with EH devices.}\label{system_model_fig}
\end{figure}

We consider a standard DFL system with a connected network comprised of $m$ wireless edge devices, as shown in Fig. \ref{system_model_fig}.
Each device $i$ is associated with a local loss function $F_i(\v{w})=  \b{E}_{\xi_i\sim \ca{D}_{i}} \left[  f_i(\v{w}; \xi_i )\right] $, where $f_i(\v{w}; \xi_i)$ denotes the sample-wise loss function that quantifies the loss of learning model $\v{w}$ on the data sample $\xi_i$, and $\ca{D}_i$ is the local dataset of device $i$.
The objective of the DFL system is to train a learning model $\v{w}$ that minimizes the global loss function $F(\v{w})$ as follows:
\begin{equation}\label{global_fl_problem}
	\begin{aligned}
		\min_{\v{w}} \; F(\v{w})=\frac{1}{m}\sum_{i=1}^{m}F_i(\v{w}).
	\end{aligned}
\end{equation}
Problem \eqref{global_fl_problem} is known as empirical risk minimization, which arises in many ML problems.

To address problem \eqref{global_fl_problem} in a decentralized manner, each device iteratively performs the following three steps:
\begin{itemize}
	\item Each device performs multiple steps of stochastic gradient descent (SGD) on its private dataset.
	\item Each device broadcasts its local model update to the neighboring devices within the network.
	\item Each device updates its local model by calculating a weighted average of the local updates received from its neighboring devices.
\end{itemize}
An iteration comprising the above three steps is defined as a time slot, indexed by $t$. The training continues until the maximum number of time slots $T$ is reached.

\subsection{Local Computation Model}

At each time slot \( t \), each device \( i \) first performs $K$ iterations of local SGD on its private dataset. Let \( \v{w}_{i,t} \) denote the local model of device \( i \) at time slot $t$, and initialize the model for local SGD by setting  \(\v{w}_{i,t}^{(0)} = \v{w}_{i,t}\). Then, for each local iteration \( k \in \{0, 1, \ldots, K-1\} \), device $i$ obtains a mini-batch $\vspace{0pt}$\( \hspace{2pt} \v{\xi}_{i,t}^{(k )} \hspace{0.5pt} \subseteq \hspace{0.5pt} \mathcal{D}_i \) from its private dataset and computes the stochastic gradient as follows:
\begin{equation}\label{eq:local_gradient}
	\v{g}_{i,t}^{(k)} = \nabla f_i(\v{w}_{i,t}^{(k)}; \v{\xi}_{i,t}^{(k)}).
\end{equation}
The local model is then updated as follows:
\begin{equation}\label{eq:local_update}
	\v{w}_{i,t}^{(k+1)} = \v{w}_{i,t}^{(k)} - \eta \, \v{g}_{i,t}^{(k)},
\end{equation}
where \( \eta > 0 \) represents the learning rate.
To account for the potential energy limitations of edge devices, we introduce a binary computation indicator \( \beta_{i,t} \in \left\lbrace 0,1\right\rbrace \). When \( \beta_{i,t} = 1 \), device $i$ executes the local SGD procedure at time slot \( t \); otherwise, the local model remains unchanged.
Accordingly, after the local update phase at time slot $t$, the local model of device \( i \) is given by
\begin{equation}
	\begin{aligned}\label{eq:update_scheduling}
		\v{w}_{i,t+\frac{1}{2}}^{} = \begin{cases}\v{w}_{i,t}^{(K)}, & \text { if } \beta_{i,t} = 1; \\ \v{w}_{i,t}^{}, & \text { if } \beta_{i,t} = 0.\end{cases}
		\end{aligned}    
\end{equation}
The local computation energy consumption for device $i$ at time slot $t$ can be expressed as \cite{yang2020energy}
\begin{equation}\label{eq:energy_consumption}
	e_{i,t}^{\text{cmp}} = \beta_{i,t} K  I_i^2  c_i  | \v{\xi}_{i} |,
\end{equation}
where \( I_i \) is the CPU cycle frequency of device \( i \), \( c_i \) denotes the number of CPU cycles to process a single data sample, and \(  | \v{\xi}_{i} | \) is the (constant) mini-batch size used in each local iteration.

\subsection{Communication Model}

In the DFL system, edge devices frequently exchange model updates via D2D communication to reach consensus on the model and collaboratively achieve global convergence.
The network topology is described by the mixing matrix, denoted as $\v{A} \in \b{R}^{N \times N}$. The $(i,j)$-th entry of $\v{A}$, denoted as $a_{i,j}$, is strictly positive if $i=j$ or devices $i,j$ are neighbors; otherwise, $a_{i,j} = 0$. Moreover, $\v{A}$ is symmetric and doubly stochastic, i.e., $\sum_{j=1}^{N}a_{i,j} =1, \forall i \in [N]$ and $\sum_{i=1}^{N}a_{i,j} =1, \forall j \in [N]$.
Let \( \lambda_i(\v{A} ) \) represent the \( i \)-th largest eigenvalue of \( \v{A}  \).
The symmetric property of $\v{A} $ indicates that its eigenvalues are real and can be sorted in the non-increasing order, i.e., 
\(
\lambda_1(\v{A} ) = 1 > \lambda_2(\v{A} ) \geq \ldots \geq \lambda_m(\v{A} ) > -1.
\) 
Furthermore, the mixing matrix $\v{A}$ also functions as the transition probability matrix of a Markov chain, whose rate of convergence to its steady-state distribution is governed by
\(
\lambda = \lambda(\v{A} ) := \max\{|\lambda_2(\v{A} )|, |\lambda_m(\v{A} )|\}.
\)

At each time slot $t$, once device $i$ completes the local SGD, it broadcasts its local model update $\Delta\v{w}_{i,t}=\v{w}_{i,t+\frac{1}{2}} - \v{w}_{i,t}$ to its neighbors.
We adopt a block-fading model for wireless channels among neighboring devices, wherein the channel fading is assumed to remain unchanged within each time slot.
Let $\tilde{h}_{i,j,t}$ denote the wireless channel from device $j$ to device $i$ at time slot $t$ with the channel gain $h_{i,j,t}=|\tilde{h}_{i,j,t}|$.
Each device transmits its local model update as a single packet, and the packet error rate cannot be ignored due to the limited transmission power.
Specifically, the packet error rate for the transmission from device $j$ to its neighboring device $i$ at time slot $t$ is given by \cite{xi2011general}
\begin{equation}\label{per}
	\begin{aligned}
		q_{i,j,t}=1-\exp\left(-\frac{ \varphi \left( \sum_{j^\prime \in \ca{N}_i \backslash j } p_{j^\prime,t} h_{i,j^\prime,t} +  \sigma^2_{i} \right) }{p_{j,t} h_{i,j,t}}\right),
	\end{aligned}    
\end{equation}
where $p_{j,t}$ denotes the transmission power, $\sigma_i^2$ denotes the total noise variance at device $i$, $\ca{N}_i$ is the neighbors of device $i$, and $\varphi$ is a constant parameter referred to as the waterfall threshold \cite{xi2011general}.
The communication energy consumption for device $i$ at time slot $t$ is given by $e_{i,t}^{\text{com}}=p_{i,t} \tau$, where $\tau$ is the transmission duration.
We further introduce $\zeta_{i,j,t}\in\left\lbrace 0,1\right\rbrace$ as the indicator for successful transmission.
In particular, $\zeta_{i,j,t}=0$ denotes that device $i$'s received local update from device $j$ at time slot $t$ is dropped due to the occurrence of packet error, while $\zeta_{i,j,t}=1$ indicates successful reception. Consequently, for all $i \neq j$, we have
\begin{equation}\label{succ_index}
	\begin{aligned}
		\zeta_{i,j,t} = \begin{cases}1, & \text { with probability } 1-q_{i,j,t}; \\ 0, & \text { with probability } q_{i,j,t}.\end{cases}
	\end{aligned}    
\end{equation}
In DFL systems, to ensure global convergence and maintain model consistency, each device $i$ updates its local model by employing a weighted average of the successfully received local updates from the neighboring devices as follows:
\begin{equation}\label{update_global}
	\begin{aligned}
		\v{w}_{i,t+1} =\v{w}_{i,t} +  \sum_{j=1}^{m} \beta_{j,t} \zeta_{i,j,t} a_{i,j}  \Delta\v{w}_{j,t}.
	\end{aligned}    
\end{equation}
{\color{black}When a neighbor device $j$ is not scheduled (i.e., $\beta_{j,t}=0$) or its transmission is unsuccessful (i.e., $\zeta_{i,j,t}=0$), device $i$ does not obtain a model update from that neighbor at time slot $t$, and the corresponding term in \eqref{update_global} is absent. In this case, device $i$ simply proceeds with the aggregation using the updates that are successfully received at this time slot.}

\subsection{Energy Harvesting Model}

The finite battery capacity of edge devices presents a substantial obstacle for maintaining long-term operation in wireless DFL systems without external battery recharging.
To mitigate this challenge, each edge device is assumed to incorporate an EH unit that harvests renewable energy from ambient sources (e.g., solar, wind, or thermoelectric sources). The harvested energy is stored in a rechargeable battery with maximum capacity $b^{\max}$. The energy collected in time slot $t$ is first stored in the battery and becomes accessible for use in the subsequent slot $t+1$.
Formally, the evolution of the battery level \(b_{i,t}\) for device $i$ is governed by
\begin{equation}\label{evo_energy}
	\begin{aligned}
		b_{i,t+1}=\min \{b_{i,t}+u_{i,t}-e_{i,t}, b^{\max}\},
	\end{aligned}    
\end{equation}
where \(u_{i,t}\) denotes the amount of harvested energy, and \(e_{i,t} = e_{i,t}^{\text{cmp}} + e_{i,t}^{\text{com}}\) represents the total energy utilization for on-device computation and communication.
To ensure that energy usage never exceeds the currently available battery level, the following energy causality constraint is imposed:
\begin{equation}\label{causality_energy}
\begin{aligned}
    0\leq e_{i,t}\leq b_{i,t}, \forall i,t.
\end{aligned}
\end{equation}
This constraint ensures that each device schedules its computations and transmissions in a manner consistent with its available energy, preventing battery depletion and supporting extended system operation.

\section{Convergence Analysis and Problem Formulation}\label{Sec_3}

In this section, we first analyze the convergence of wireless DFL with EH devices, showing how partial device participation and transmission packet drops influence the learning performance. We then formulate a joint device scheduling and power control problem to accelerate the convergence.

\subsection{Convergence Analysis}
To enable the convergence analysis, we adopt the following conventional assumptions, which are commonly used in the existing FL literature \cite{liu2022hierarchical, 9210812, wan2021convergence,wang2024communication,yang2024unleashing,zhang2023online,zhang2024federated}.

\begin{assump}\label{ass_smoothness_assump}
	The local loss function $F_i(\v{w})$ is $L$-smooth. For any $\v{w}, \v{w}^\prime \in \b{R}^d$, $i \in [m]$, we have
	\begin{align}\label{smoothness}
		\begin{split}
			\left\| \nabla F_i (\v{w}) - \nabla F_i (\v{w}^\prime)\right\|  \leq L \|\v{w}-\v{w}^\prime\|.
		\end{split}
	\end{align}
\end{assump}

\begin{assump}\label{ass_bounded_var_assump}
	The stochastic gradient of each device has a bounded local variance, i.e., for any $\v{w}\in \b{R}^d $, $i\in [m]$, we have
	\begin{equation}\label{bounded_grad_eq}
		\begin{aligned}
			\mathbb{E}_{\ca{\xi} \sim \ca{D}_i}\left[ \left\| \nabla f_i(\v{w}, \ca{\xi}) - \nabla F_i(\v{w}) \right\|^2 \right] \leq \sigma_l^2.
		\end{aligned}    
	\end{equation}
	The local loss functions have a dissimilarity bound that characterizes their heterogeneity. For any $\v{w}\in \b{R}^d $, $i\in [m]$, we have
	\begin{equation}\label{bounded_heto_eq}
		\begin{aligned}
			\left\|\nabla F_i(\v{w}) - \nabla F(\v{w}) \right\|^2 \leq \sigma_g^2.
		\end{aligned}    
	\end{equation}
\end{assump}

\begin{assump}\label{ass_bounded_G_assump}
	The expected squared norm of the stochastic gradients is uniformly bounded, i.e., for any $\v{w}\in \b{R}^d $, $i\in [m]$, we have
	\begin{equation}\label{bounded_G}
		\begin{aligned}
			\mathbb{E}\left[ \left\| \nabla f_i(\v{w}, \ca{\xi}) \right\|^2 \right] \leq G^2.
		\end{aligned}    
	\end{equation}
\end{assump}

Based on the assumptions above, we establish the convergence bound of wireless DFL with EH for general nonconvex loss functions.
Specifically, we analyze the convergence of the average model, which is defined as
\begin{equation}\label{w_bar}
	\begin{aligned}
		\v{\bar w}_t = \frac{1}{m}\sum_{i=1}^{m}  \v{w}_{i,t}.
	\end{aligned}    
\end{equation}

\begin{thm}\label{thm_convergence_bound}
Under Assumptions~\ref{ass_smoothness_assump}, \ref{ass_bounded_var_assump}, \ref{ass_bounded_G_assump}, if we set $\eta=\tfrac{\sqrt{m}}{64LK\sqrt{T}}$, it follows that
\begin{equation}
	\begin{aligned}\label{convergence_bound}
		&\frac{1}{T}\sum_{t=1}^{T} \b{E} \left[  \left\| \nabla F\left( \bar{\v{w}}_{t} \right) \right\|^2  \right] \\		
		&   \leq  \frac{ 256 L \left( F\left( \bar{\v{w}}_{1} \right) - F\left( \v{w}^{*}\right) \right) }{ \sqrt{mT}}+  \frac{\sqrt{m} C_1}{2 K T^{\frac{1}{2}}} + \frac{m C_1 }{256 K T} \\
		&  + \frac{ \sqrt{m}\left(C_1 + 4 K G^2 \right) }{128 ( 1 - \lambda)K T^{\frac{3}{2}}} +  \frac{m\left(  C_1 + 4 K G^2 \right) }{128^2( 1 - \lambda)KT^2}\\
		&  +\! \hspace{-1pt}\frac{4  \!\left( \hspace{-1pt} KL \!+\! \sqrt{ K}\hspace{0pt} \right) \! \hspace{0pt}G^2}{(m-1)KT} \! \sum_{t=1}^{T}  \sum_{j=1}^{m} \sum_{ i\neq j}^{m}( \hspace{-0.5pt} \left( \hspace{-0.5pt} (\hspace{-0.5pt}m \!-\! 1\hspace{-0.5pt}) a_{i,j} q_{i,j,t}\!-\! 1 \hspace{-1pt}\right)\hspace{-1pt}  \beta_{j,t} \!+\! 1  \hspace{-0.5pt} ),\\
	\end{aligned}
\end{equation}
where $C_1 = \sigma_l^2 +4 K \sigma_g^2$.
\end{thm}
\begin{proof}
{The proof is presented in Appendix A.}
\end{proof}

\begin{rem}

The first five terms on the right-hand side (RHS) of \eqref{convergence_bound} arise from the initial error (i.e., the gap between the initial model $\bar{\v{w}}_{1}$ and the optimal model $\v{w}^{*}$) as well as other learning-related factors (e.g., device heterogeneity and data heterogeneity).
These terms gradually diminish as the number of time slots $T$ increases.
In contrast, the final term represents the persistent error introduced by partial device participation and transmission packet drops, and this error does not vanish as $T$ grows.
In particular, the convergence performance of wireless DFL with EH can be enhanced by increasing device participation and raising their transmission power.
Nevertheless, both device scheduling and transmission power are constrained by limited and uncertain harvested energy in practice, highlighting the importance of designing effective scheduling and transmission policies to balance energy constraints and convergence performance.

\end{rem}


\subsection{Problem Formulation}
In this paper, we focus on optimizing the device scheduling $\{\beta_{i,t}\}$ and transmission power $\{p_{i,t}\}$ to accelerate the convergence of wireless DFL with EH devices.
However, directly minimizing the global loss is intractable since it cannot be expressed in closed form with respect to the scheduling and transmission strategies.
To mitigate this challenge, we approximate the actual global loss using the convergence bound derived in Theorem \ref{thm_convergence_bound}.
In particular, we ignore the first five terms on the RHS of \eqref{convergence_bound}, as they are independent of the device scheduling and transmission power and go to zero\hspace{13pt} as\hspace{0.5pt} $T$ \hspace{0.5pt}increases. \hspace{0.5pt}Hence, \hspace{0.5pt}the \hspace{0.5pt}resulting \hspace{0.5pt}optimization \hspace{0.5pt}problem \hspace{0.5pt}is 

\hspace{-9.5pt}formulated as follows:
		\begin{align}\label{P1_orig}
		&\!\!\underset{\{\v{\beta}_{t},\v{p}_t\}}{\text{min}}  \frac{4 \left(  K L + \sqrt{K} \right)  G^2}{(m-1)KT} \nonumber \\
		&~~~~~~~ \times \sum_{t=1}^{T}  \sum_{j=1}^{m} \sum_{ i\neq j}^{m} ( \left(  (m - 1) a_{i,j} q_{i,j,t}- 1\right)  \beta_{j,t} + 1 ) \nonumber \\
		&~~\text{s.t.} \;  ~ \text{C1:\;} 0\leq e_{i,t}\leq b_{i,t}, \forall i,t, 
	\end{align}
where $\v{\beta}_{t}=\{\beta_{1,t},\ldots,\beta_{m,t}\}$ and $\v{p}_{t}=\{p_{1,t},\ldots,p_{m,t}\}$.
Constraint C1 enforces the energy causality requirement, ensuring that the energy consumed by each device does not exceed its available battery level.

To further simplify problem \eqref{P1_orig}, we express the device scheduling indicator $\beta_{i,t}$ in terms of the transmission power $p_{i,t}$.
Concretely, device $i$ is regarded as scheduled if and only if its transmit power $p_{i,t}$ is positive, i.e., $\beta_{i,t}=\mathbbm{1}_{\left\lbrace p_{i,t}> 0\right\rbrace }$.
Then, problem \eqref{P1_orig} can be reformulated as follows:
\begin{align}\label{P1_reformulated}
	&\underset{\{\v{p}_t\}}{\text{min}}  \frac{4 \left(  K L + \sqrt{K} \right)  G^2}{K T } \nonumber \\
	&~~~~ \times \! \sum_{t=1}^{T}\hspace{-1pt} \sum_{j=1}^{m}\hspace{-1pt} \sum_{i\neq j}^{m} \! a_{i,j} \!\! \left( \!\hspace{-1pt}1\!-\!\exp\! \!\left(\!\hspace{-1pt}-\frac{ \varphi \! \left(\! \sum_{\hspace{-1pt}j^{\hspace{-0.5pt} \prime}\hspace{-1pt} \in \ca{N}_i \hspace{-1pt}\backslash j } \hspace{-1pt} p_{j^{\hspace{-0.5pt}\prime}\hspace{-1pt},t} h_{i,j^{\hspace{-0.5pt} \prime} \hspace{-1pt},t} \!+\!  \sigma^2_{i} \right) }{p_{j,t} h_{i,j,t}}\!\right)\!\!\! \right)   \nonumber \\
	&~\text{s.t.} \;  ~ \text{C1}.
\end{align}

The reformulated problem~\eqref{P1_reformulated} remains challenging due to two main sources of uncertainty: (i) stochastic wireless channels, whose slow-fading behavior can be modeled as a finite-state Markov process, and (ii) the stochastic amount of harvested energy, which directly affects scheduling and power control decisions.
These stochastic processes unfold sequentially over time, with channel states and harvested energy levels only becoming available at each time slot, thereby rendering traditional offline optimization methods unsuitable.
Observing that both battery levels and channel conditions can be characterized by Markov processes, we model problem \eqref{P1_reformulated} as a multi-agent MDP, in which devices select their transmission power based on their local states (channel condition and battery level). In the following section, we illustrate how this MDP framework yields an optimal policy for device scheduling and power control.

\section{MDP Formulation and Decentralized Policy Iteration Algorithm}\label{Sec_MDP}

In this section, we develop a multi-agent MDP framework to derive the optimal transmission policy.
Then, we propose a fully decentralized policy iteration algorithm that requires only local state information from two-hop neighboring devices. We further present a theoretical analysis demonstrating that the proposed decentralized algorithm achieves asymptotically optimal performance.

\subsection{MDP Formulation}

To account for the temporal correlation of realistic wireless channels, we employ a finite-state Markov channel model, which is widely adopted to approximate various analytic and empirical fading models, such as indoor channels and Rayleigh fading channels \cite{sadeghi2008finite}.
Specifically, the wireless channel is discretized into $N_h$ states, with the corresponding channel gains represented by $\ca{H}={H_{1},H_{2},\ldots,H_{N_h}}$.
Without loss of generality, we assume that $H_k < H_{k^\prime}$ whenever $k < k^\prime$.
Then, the transition probability matrix for the $i$-th device's channel is given by
\begin{equation}\label{def_phi}
	\begin{aligned}
		\Psi_i=\left[\begin{array}{lll}
			\psi_{i}^{(1,1)} & \cdots & \psi_{i}^{(1,N_h)} \\
			\vdots & \ddots & \vdots \\
			\psi_{i}^{(N_h,1)} & \cdots & \psi_{i}^{(N_h,N_h)}
		\end{array}\right]
	\end{aligned}    
\end{equation}
with its $(k,k^\prime)$-th entry $\psi_{i}^{(k,k^\prime)}=\b{P}(h_{i,t+1}=H_{k^\prime}|h_{i,t}=H_{k})$.
Furthermore, the harvested energy $u_{i,t}$ is assumed i.i.d. across time slots, following the probability distribution $\b{P}(U_i)$ determined by the external environment of device $i$.
Similarly, the battery level of each device is discretized into a finite set denoted by $\ca{B}$.
By representing both the wireless channel dynamics and the battery evolution as homogeneous finite-state Markov chains, problem~\eqref{P1_reformulated} can be reformulated as a multi-agent MDP.
This MDP is characterized by the tuple $\left(\ca{S}, \ca{P}, \b{P}(\v{s}_{t+1}|\v{s}_{t},\v{p}_t), c(\v{s}_t,\v{p}_t)\right) $, where each element is described in the following.

\begin{itemize}
	\item \textbf{States}: Let $s_{i,t}=(\left\lbrace h_{i,j,t}  \right\rbrace ,b_{i,t}) \in \ca{S}_i$ represent the state of device $i$ at time slot $t$, where $\ca{S}_i$ is the device-specific state space. Accordingly, the global system state is given by $\v{s}_{t} = (s_{1,t},\ldots,s_{m,t})\in \ca{S}$, where $\ca{S}=\ca{S}_1\times\ldots\times\ca{S}_m$ denotes the global state space.
	\item \textbf{Actions}:
	Since practical wireless devices generally support only a finite number of transmission power levels, each device $i$ is assumed to determine its transmission power $p_{i,t}$ from a finite discrete action space $\ca{P}_i$.
	Consequently, the joint action space of all devices is expressed as $\ca{P}=\ca{P}_1\times\ldots\times\ca{P}_m$.
	
	\item \textbf{Transition probabilities}:
	At each time slot $t$, the transition probability $\b{P}(\v{s}_{t+1}|\v{s}_{t},\v{p}_t)$ specifies the distribution of the next state $\v{s}_{t+1}$ conditioned on the current state–action pair $(\v{s}_{t},\v{p}_t)$. It can be factorized as
	\begin{equation}\label{trans_prob}
		\begin{aligned}
			&\b{P}(\v{s}_{t+1}|\v{s}_{t},\v{p}_t)\\
			=&\prod_{i=1}^{m}\b{P}(s_{i,t+1}|s_{i,t},p_{i,t})\\
			=&\prod_{i=1}^{m}\b{P}(h_{i,t+1}|h_{i,t})\b{P}(b_{i,t+1}|h_{i,t},b_{i,t},p_{i,t}).
		\end{aligned}    
	\end{equation}
	Given the battery evolution in \eqref{evo_energy},  the first-order conditional probability for the battery level $b_{i,t}$ can be expressed as follows:
	\begin{equation}\label{trans_prob_b}
		\begin{aligned}
			\b{P}(b_{i,t+1}|h_{i,t},b_{i,t},p_{i,t})\quad\quad\quad\quad\quad\quad\quad\quad\quad\quad\quad\quad~~&\\
			=\begin{cases}\b{P}\left(U_i = b_{i,t+1}- b_{i,t} + e_{i,t} \right) ,&\text{if  } b_{i,t+1} < b^{\max};
				\\ \b{P}\left(U_i \geq b^{\max}- b_{i,t} + e_{i,t} \right),&\text{if  } b_{i,t+1} = b^{\max}.\end{cases} \nonumber
		\end{aligned}    
	\end{equation}
	\item \textbf{One-step cost}: Following the objective function in problem \eqref{P1_reformulated}, we define a cost function that characterizes the penalty arising from partial device participation and transmission packet drops. Formally, according to \eqref{P1_reformulated}, the one-step cost at time slot $t$ is given by
	\begin{equation}
		\begin{aligned}\label{cost}
			&c(\v{s}_{t},\v{p}_{t}) = \frac{4  \left(  K L + \sqrt{K} \right)  G^2}{K T} \\
			&\times  \sum_{j=1}^{m} \sum_{i\neq j}^{m} \! a_{i,j} \!\! \left( \!\hspace{-1pt}1\!-\!\exp\! \!\left(\!\hspace{-1pt}-\frac{ \varphi \! \left(\! \sum_{\hspace{-1pt}j^{\hspace{-0.5pt} \prime}\hspace{-1pt} \in \ca{N}_i \hspace{-1pt}\backslash j } \hspace{-1pt} p_{j^{\hspace{-0.5pt}\prime}\hspace{-1pt},t} h_{i,j^{\hspace{-0.5pt} \prime} \hspace{-1pt},t} \!+\!  \sigma^2_{i} \right) }{p_{j,t} h_{i,j,t}}\!\right)\!\! \right)\!.
		\end{aligned}
	\end{equation}
\end{itemize}

The goal of this MDP is to devise an optimal policy that achieves the minimum expected cumulative cost.
Formally, the policy is defined as $\pi=\{\pi_t : \ca{S} \rightarrow \ca{P}\}_{t=1}^T$, where $\pi_t$ maps the global system state $\v{s}_t$ to the transmission power vector $\v{p}_t$. For a specified policy $\pi$, its expected cumulative cost is expressed as
\begin{equation}\label{def_J}
	\begin{aligned}
		J(\pi)=\b{E}_{\pi}\left[ \sum_{t=1}^{T} c(\v{s}_{t},\v{p}_{t}) \bigg| \v{s}_{1}\right],\\
	\end{aligned}
\end{equation}
where $\v{s}_{1}$ denotes the initial global state.
By formulating the expected cumulative cost in this form, problem \eqref{P1_reformulated} can be recast as finding the optimal policy $\pi^*$ that minimizes the expected cumulative cost as follows.

\vspace{-3pt}
\begin{problem}\label{Problem1}
	\textup{(Expected Cumulative Cost Minimization Problem)}
	\begin{equation}
		\begin{aligned}
			\pi^*=\underset{\pi}{\textup{argmin}}~&J(\pi)\\
			\textup{s.t.} \quad & \textup{C1}.
		\end{aligned}    
	\end{equation}
\end{problem}
\vspace{-2pt}

To derive the optimal solution for Problem \ref{Problem1}, a straightforward approach is to apply classical (single-agent) dynamic programming methods, e.g., value iteration or policy iteration.
However, applying these methods to Problem \ref{Problem1} necessitates a centralized controller, which collects the local states of all devices, computes optimal actions by enumerating all state-action pairs, and then disseminates these actions to devices. Specifically, let $Q_t(\v{s}_t,\v{p}_t)$ denote the Q-value function, which represents the expected total cost incurred when action $\v{p}_t$ is executed in state $\v{s}_t$ at time slot $t$. For the final time slot $T$, it follows that
\begin{equation}\label{bellman_equ_terminal}
	\begin{aligned}
		Q_{T}(\v{s}_T,\v{p}_T)=&c(\v{s}_T,\v{p}_T),
	\end{aligned}    
\end{equation}
since no future steps remain beyond time slot $T$.
For each time slot $t \in \left\lbrace T-1, \ldots, 1 \right\rbrace $, the Q-value function $Q_{t}(\v{s}_t,\v{p}_t)$ can be computed via the Bellman recursion as follows:
\begin{equation}\label{bellman_equ}
	\begin{aligned}
		&Q_{t}(\v{s}_t,\v{p}_t)\!=\! c(\v{s}_t,\v{p}_t)\!+\!\b{E}\!\left[\min_{\v{p}_{t+1}\in \ca{P}_{t+1}}\! Q_{t+1}(\v{s}_{t+1},\v{p}_{t+1})\Big|\v{s}_t,\v{p}_t\!\right]\!,\\
	\end{aligned}    
\end{equation}
where $\ca{P}_{t}=\left\lbrace \v{p}_{t}  \in \ca{P}:0\leq e_{i,t}\leq b_{i,t}, \forall i \in [m]\right\rbrace$.
Intuitively, the Q-value function $Q_t(\v{s}_t,\v{p}_t)$ accumulates the one-step cost $c(\v{s}_t,\v{p}_t)$ plus the expected future cost over subsequent time slots.
Upon completing the backward recursion (i.e., once $Q_{1}(\cdot,\cdot), \ldots, Q_{T}(\cdot,\cdot)$ have been fully computed), the optimal policy $\pi^* = \left\lbrace \pi_t^* \right\rbrace $ is derived by choosing the joint action $\v{p}_t \in \ca{P}_{t}$ that minimizes $Q_{t}(\v{s}_t,\v{p}_t)$ as follows:
\begin{equation}\label{opt_policy}
	\begin{aligned}
		&\pi_t^*(\v{s}_t)= \argmin_{\v{p}_t \in \ca{P}_{t}}Q_{t}(\v{s}_t,\v{p}_t).\\
	\end{aligned}    
\end{equation}
For the online implementation phase, at each time slot $t$, the centralized controller collects the global state $\v{s}_t$ and determines the optimal joint transmit power vector $\v{p}_t$ accordingly.

Although the centralized method outlined above is conceptually straightforward and guarantees the globally optimal solution to Problem \ref{Problem1}, it poses several critical challenges in large-scale EH-enabled DFL systems:

$\bullet$  Exponential Complexity: The joint state-action space scales exponentially with the number of devices $m$, rendering global dynamic programming algorithms computationally intractable for large networks.

$\bullet$  Communication Overhead: Repeatedly gathering local states and transmitting globally determined actions can lead to substantial overhead on resource-constrained EH devices.

$\bullet$  Single Point of Failure: A fully centralized control paradigm reintroduces the risk of relying on one node for global coordination, undermining the system’s resilience.

These limitations motivate a shift toward \emph{fully decentralized} solutions, in which each device manages its local state and coordinates directly with its neighbors instead of relying on a single central controller.
In the following subsection, we propose a decentralized policy iteration algorithm, which alleviates the exponential complexity of enumerating the global state-action space, reduces communication overhead, and eliminates dependence on a single control node.

\subsection{Decentralized Policy Iteration Algorithm}
In this subsection, we propose a decentralized policy iteration algorithm, in which each device determines its transmission power using only the states of its two-hop neighboring devices.
This localization strategy ensures that both the computational complexity and communication overhead scale more favorably with the network size.

\subsubsection{Localized Policy}
We start by defining the two-hop localized policy of each device. Let $\hat{\ca{N}}_j = \{j\} \cup \ca{N}_j \cup \left(\cup_{i\in \ca{N}_j} \ca{N}_i  \right )$ denote the two-hop neighborhood of device $j$ (including device $j$ itself).
Then, the two-hop localized state and action of device $j$ are given by $\v{s}_{\hat{\ca{N}}_j ,t} = \left\lbrace \v{s}_{i,t} : i \in \hat{\ca{N}}_j  \right\rbrace $ and $\v{p}_{\hat{\ca{N}}_j ,t} = \left\lbrace p_{i,t}: i \in \hat{\ca{N}}_j  \right\rbrace $, respectively. 
Formally, the two-hop localized policy of each device $j$ is defined as $\pi_{j} = \left\lbrace \pi_{j,t}( p_{j,t} |\v{s}_{\hat{\ca{N}}_j , t})\right\rbrace $, which maps the two-hop localized state $\v{s}_{\hat{\ca{N}}_j , t}$ of device $j$ to a probability distribution supported on the set of local actions $p_{j,t} \in \ca{P}_j$ at each time slot $t$.
This localized policy enables each device to operate solely based only on the local information from its two-hop neighborhood.
{\color{black}
The proposed algorithm adopts a fixed two-hop neighborhood to balance coordination benefit and communication overhead. As a promising future direction, the neighborhood scope can be adjusted according to network conditions such as node degree and link reliability. For example, a device may start with a smaller neighborhood and expand the scope only when additional coordination yields clear performance gains, which can be useful in large-scale or sparse deployments.}

Building on the above localized representation, we further define a localized cost function and a corresponding localized Q-value function. 
Specifically, the localized cost function of device $j$ is defined as 
	\begin{equation}
	\begin{aligned}\label{thelocalizedcostfunction}
		&c_j(\v{s}_{\hat{\ca{N}}_j ,t},\v{p}_{\hat{\ca{N}}_j ,t} )= \frac{4 \left(  K L + \sqrt{K} \right)  G^2}{K T } \\
		&~~~~\times  \hspace{-1pt} \sum_{i=1,i\neq j}^{m} a_{i,j} \!\! \left( \!\hspace{-1pt}1\!-\!\exp\! \!\left(\!\hspace{-1pt}-\frac{ \varphi \left(\! \sum_{\hspace{-1pt}j^{\hspace{-0.5pt} \prime}\hspace{-1pt} \in \ca{N}_i \hspace{-1pt}\backslash j } \hspace{-1pt} p_{j^{\hspace{-0.5pt}\prime}\hspace{-1pt},t} h_{i,j^{\hspace{-0.5pt} \prime} \hspace{-1pt},t} +  \sigma^2_{i} \right) }{p_{j,t} h_{i,j,t}}\!\right)\!\! \right) \hspace{-1pt},
	\end{aligned}
\end{equation}
which depends on the two-hop localized state and action of device $j$.
Then, the overall cost of the system can be decomposed into the sum of these localized costs as $c(\v{s}_{t},\v{p}_{t}) = \sum_{j=1}^{m}c_j(\v{s}_{\hat{\ca{N}}_j ,t},\v{p}_{\hat{\ca{N}}_j ,t} )$.
The localized Q-value function of device $j$ is similarly defined as $Q_{j,t}(\v{s}_{\hat{\ca{N}}_j ,t},\v{p}_{\hat{\ca{N}}_j ,t} )$, which is recursively computed in the following algorithm design.

\subsubsection{Algorithm Development}

The complete procedure of the proposed decentralized policy iteration algorithm is outlined in Algorithm \ref{D_algorithm}.
The algorithm proceeds in backward order from the final time slot $T$ down to the first time slot. At time slot $T$, each device $i$ initializes its localized Q-value function $Q_{i,T}(s_{\mathcal{\hat N}_i,T}, p_{\mathcal{\hat N}_i,T})$ to the instantaneous localized cost $c_{i,T}(s_{\mathcal{\hat N}_i,T}, p_{\mathcal{\hat N}_i,T})$ (Lines~3--4). For any earlier time slot $t < T$, each device set its localized Q-value function by using a standard Bellman recursion:
\begin{equation}
	\begin{aligned}
		& Q_{i,t}(\v{s}_{\hat{\ca{N}}_i,t}, \v{p}_{\hat{\ca{N}}_i,t}) 
		= c_{i,t}(\v{s}_{\hat{\ca{N}}_i,t},\v{p}_{\hat{\ca{N}}_i,t})\\
		&+ \min_{\v{p}_{\hat{\ca{N}}_i,t+1}} \!\b{E}\!\left[ Q_{i,t+1}(\v{s}_{\hat{\ca{N}}_i,t+1}, \v{p}_{\hat{\ca{N}}_i,t+1}) \Big| \v{s}_{\hat{\ca{N}}_i,t}, \v{p}_{\hat{\ca{N}}_i,t} \right]  ,
	\end{aligned}    
\end{equation}
as indicated in Line~6.

\begin{algorithm}[tbp]
	\caption{Decentralized Policy Iteration Algorithm} \label{D_algorithm}
	\textbf{Initialize:} $t \leftarrow T$ \;
	
	\While{$t \geq 1$} 
	{
		\uIf{$t=T$}{
			For each device $i$ and $ (\v{s}_{\hat{\ca{N}}_i,T}, \v{p}_{\hat{\ca{N}}_i,T}) $,
			\vspace{-2pt}
			\begin{align}
			Q_{i,T}(\v{s}_{\hat{\ca{N}}_i,T}, \v{p}_{\hat{\ca{N}}_i,T}) =
			c_{i,T}(\v{s}_{\hat{\ca{N}}_i,T},\v{p}_{\hat{\ca{N}}_i,T}). \nonumber
			\end{align}
		}
		\Else{
			For each device $i$ and $ (\v{s}_{\hat{\ca{N}}_i,t}, \v{p}_{\hat{\ca{N}}_i,t}) $,
			\vspace{-2pt}
			\begin{equation}
				\hspace{-7pt}
				\begin{aligned}
					&Q_{i,t}(\v{s}_{\hat{\ca{N}}_i,t}, \v{p}_{\hat{\ca{N}}_i,t}) 
					= c_{i,t}(\v{s}_{\hat{\ca{N}}_i,t},\v{p}_{\hat{\ca{N}}_i,t})\\
					&+ \min_{\v{p}_{\hat{\ca{N}}_i,t+1}} \!\b{E}\!\left[ Q_{i,t+1}(\v{s}_{\hat{\ca{N}}_i,t+1}, \v{p}_{\hat{\ca{N}}_i,t+1}) \Big| \v{s}_{\hat{\ca{N}}_i,t}, \v{p}_{\hat{\ca{N}}_i,t} \right] \!. \nonumber
				\end{aligned}    
			\end{equation}
		}
		\mbox{Initialize $\pi_{i,t}^{1}$ and $Q_{i,t}^1(\v{s}_{\hat{\ca{N}}_i}, \v{p}_{\hat{\ca{N}}_i})\!=\! Q_{i,t}(\v{s}_{\hat{\ca{N}}_i}, \v{p}_{\hat{\ca{N}}_i}) , \forall i \!$ \;}
		
		\hspace{-3pt}
		\For{$r=1,\ldots R$}
		{
			For each device $i$ and $ (\v{s}_{\hat{\ca{N}}_i}, \v{p}_{\hat{\ca{N}}_i}) $, update
			\hspace{-20pt}
			\vspace{-2pt}
			\hspace{-20pt}
			\begin{align}
			\hspace{-7pt}
			\!\! Q_{i,t}^{r+1}(\v{s}_{\hat{\ca{N}}_i}, \v{p}_{\hat{\ca{N}}_i}) \!\hspace{-1pt}=\hspace{-1pt}\! \frac{1}{|   \hat{\ca{N}}_i   |}\!\sum_{j\in \hat{\ca{N}}_i}\! \hspace{-1pt} Q_{j,t}^r \! \hspace{-1pt} \left(\! \left[ \v{s}_{\hat{\ca{N}}_i}\right]_{\hat{\ca{N}}_j} \!,\hspace{-1pt} \left[ \v{p}_{\hat{\ca{N}}_i}\right]_{\hat{\ca{N}}_j} \!\right)\! \hspace{-1pt}, \nonumber
			\end{align} 
			\vspace{-7pt}
			
			\hspace{-0pt}and
			\begin{equation}
				\hspace{-13pt}
				\begin{aligned}\label{pi_improve}
					&\pi_{i,t}^{r+1}\left(  p_{i} \big|\v{s}_{\hat{\ca{N}}_i}\right) = \\
					&\! \frac{ \exp\!\left( -  \gamma \b{E}_{\! p_j \sim \pi_{j,t}^{r}\!\left(\!\!\left[ \v{s}_{\hat{\ca{N}}_i}\right]_{\hat{\ca{N}}_j}\!\!\right), j \in \hat{\ca{N}}_i\! \backslash \! i} \left[Q_{i,t}^{r+1}(\v{s}_{\hat{\ca{N}}_i}, \v{p}_{\hat{\ca{N}}_i}) \right] \right) }{\sum_{p_i}\!\! \exp\!\!\left(\! \!- \gamma\b{E}_{\! p_j \sim \pi_{j,t}^{r}\!\left(\!\!\left[ \v{s}_{\hat{\ca{N}}_i}\right]_{\hat{\ca{N}}_j}\!\!\right)\!, j \in \hat{\ca{N}}_i\! \backslash \! i}\! \left[\!Q_{i,t}^{r+1}(\v{s}_{\hat{\ca{N}}_i}, \v{p}_{\hat{\ca{N}}_i})\! \right]\!\! \! \right) }  \nonumber
				\end{aligned} 
			\end{equation}
			\vspace{-2pt}
		} 
		
		$\pi_{i,t} = \pi_{i,t}^{R+1}$\;
		
		$t \leftarrow t-1$\;
	}
	\mbox{\textbf{Output:} Localized policy $ \pi_{\textup{loc}} = \bigl\{\pi_{i,t} |  i \in [m], t \in [T]\bigr\} $ \;}
\end{algorithm}

Once the localized Q-value functions $Q_{i,t}(\v{s}_{\hat{\ca{N}}_i,t}, \v{p}_{\hat{\ca{N}}_i,t}) $ have been computed, the algorithm initializes the localized policy $\pi_{i,t}^1$ (Line~8) and proceeds with $R$ policy-improvement iterations (Lines~9--13). In each iteration $r$, each device $i$ aggregates the localized Q-values from its two-hop neighboring devices $\mathcal{\hat N}_i$ and then updates its own Q-value function $Q_{i,t}^{r+1}(\v{s}_{\hat{\ca{N}}_i}, \v{p}_{\hat{\ca{N}}_i})$ by averaging the two-hop neighbors’ Q-values (i.e., $Q_{j,t}^{r}, j\in \hat{\mathcal{N}}_i$) (Line 10).
This step enforces partial consensus among neighbors, effectively blending local information across the network.
Here, a ctritical challenge is that device $j \in \hat{\ca{N}}_i$ may in turn have its own two-hop neighbors $\hat{\ca{N}}_j$, which are not contained in $\hat{\ca{N}}_i$.
To address this, for the nodes $l \in \hat{\mathcal{N}}_j  \backslash \hat{\mathcal{N}}_i$, we assign a default state ${s}_{\textup{default}}$ and a default action $p_{\textup{default}}$ to replace their states and actions, respectively.
Formally, we define an extension operator for the local states as follows:
	\begin{equation}
	\begin{aligned}
		\left[ \v{s}_{\hat{\ca{N}}_i}\right]_{j}= \begin{cases} {s}_j, & \text { if  } ~j \in \hat{\ca{N}}_i; \\ {s}_{\textup{default}}, & \text { if  } ~j \notin \hat{\ca{N}}_i.\end{cases}
	\end{aligned}    
\end{equation}
A similar definition is used for the local actions as follows:
	\begin{equation}
	\begin{aligned}
		\left[ \v{p}_{\hat{\ca{N}}_i}\right]_{j}= \begin{cases} {p}_j, & \text { if  } ~j \in \hat{\ca{N}}_i; \\ {p}_{\textup{default}}, & \text { if  } ~j \notin \hat{\ca{N}}_i.\end{cases}
	\end{aligned}    
\end{equation}
We note that the default state $s_{\textup{default}}$ and the default action $p_{\textup{default}}$ are fixed and will not affect the final convergence result.
Next, each device refines its localized policy $\pi_{i,t}^{r+1}$ through a multiplicative-weights update (Line~12), giving higher probability to actions that minimize the aggregated localized Q-value.

By combining backward recursion with repeated local policy improvements, the algorithm achieves locally optimized decisions at each device, while the Q-value information propagated through two-hop neighborhoods collectively steers the system toward a near-global optimum.

{\color{black}

\subsubsection{Complexity and Communication Overhead}
In the offline policy computation stage, the main computational cost arises from computing the localized Q-value functions and the corresponding localized policies in Algorithm~\ref{D_algorithm}. Due to the localized structure, each device only enumerates its two-hop neighborhood state--action space, and thus the computational complexity depends on the two-hop neighborhood size rather than the total number of devices in the network. The communication overhead in this offline stage is dominated by the $R$ policy-improvement iterations, where each device exchanges localized Q-values and localized policies only with its two-hop neighbors. This involves localized peer-to-peer message passing and avoids any network-wide information collection. By contrast, centralized dynamic programming baselines operate over the global joint state--action space, which leads to exponential computational complexity in the number of devices and requires centralized access to all device states during policy computation.

In the online policy execution stage, each device executes the learned localized policy at each time slot via a lightweight lookup based on its current two-hop neighborhood state, and no additional network-wide communication is required. By contrast, centralized benchmarks typically need to repeatedly collect global state information and compute and distribute the joint action at every time slot, which incurs substantial per-slot communication overhead and limits scalability in large networks.

The above complexity discussion assumes a moderate two-hop neighborhood size. In dense network deployments, the number of two-hop neighbors can grow significantly, which enlarges the localized state--action space and may impose substantial computation and memory burdens when using a tabular implementation of the localized Q-values and policies. Let $d$ denote the maximum number of one-hop neighbors. In the worst case, the number of devices within two hops can scale in the order of $\mathcal{O}(d^2)$, and the localized state--action space can grow rapidly with $d$. In practical wireless DFL systems, however, the effective neighborhood size is often limited by physical-layer and energy constraints, since only a subset of nearby links with sufficiently reliable channel quality can sustain D2D communication and the available energy budget further restricts the number of neighbors that can be interacted with. For denser scenarios, the framework can be made more scalable by constructing a sparse mixing topology that prunes weak or low-utility links to limit the neighborhood size, and by replacing the tabular Q-function and policy with function approximation (e.g., a lightweight value function or policy network) to avoid storing and updating all state--action combinations.

}

{\color{black}
\begin{rem}
The exchange of localized Q-values and localized policies among two-hop neighbors incurs additional communication overhead and energy consumption. In practice, this overhead is typically moderate because the localized state and action spaces are discretized and the resulting localized Q-values and localized policies can be stored and shared in a finite lookup table form whose size is determined by each device’s two hop neighborhood rather than the entire network.
Moreover, the exchange of localized Q-values and policies is carried out in a one time offline policy computation stage. After the policy is obtained, devices only execute the learned localized policy online and do not require repeated policy learning or additional control message exchanges at each time slot. The offline computation is invoked again only when the environment changes materially such as topology changes or changes in channel and energy harvesting statistics.

\end{rem}}

\subsection{Performance Analysis}

In this subsection, we theoretically analyze the proposed decentralized policy iteration algorithm, demonstrating that it converges asymptotically to the globally optimal policy.
Intuitively, the performance gap between the proposed decentralized policy iteration algorithm and its centralized counterpart arises from an additional error introduced by truncating both the Q-value function and the policy to each device’s two-hop neighborhood, rather than the global network.
As will be established later, this localized approximation error vanishes as the number of policy improvement iterations $R$ increases.
Formally, we derive an upper bound for the gap between the expected cumulative costs under the localized policy $\pi_{\textup{loc}} $ generated by Algorithm \ref{D_algorithm} and the optimal policy $\pi^{*}$, as shown in the following theorem.
\begin{thm}\label{thm_convergence}
	By taking the algorithm parameter $\gamma \leq \frac{1}{32 m (L+1) G^2   | \ca{P} |^2} $, the gap between the expected cumulative costs under the localized policy $\pi_{\textup{loc}}$ generated by Algorithm \ref{D_algorithm} and the optimal policy $\pi^{*}$ is upper bounded by
	\begin{equation}\label{convergence_bound_MDP}
		\begin{aligned}
			&J\left(\pi_{\textup{loc}}\right) - J\left(\pi^*\right) \\
			\leq & \frac{4 m \left(  K L + \sqrt{K} \right)  G^2 D^R}{K }  \\
			& \times \sum_{t=1}^{T}  \sum_{i=1}^{m} \sup_{\v{s}} \textup{TV} \left( \pi_{i,t}^{1}\left(  p_{i} \big|\v{s}_{\hat{\ca{N}}_i}\right) , \pi_{i,t}^{*}\left(  p_{i} |\v{s}\right)  \right) \\
			= & \ca{O}\left(D^R\right)  .
		\end{aligned}    
	\end{equation}
	where $D = 32 \gamma m (L+1) G^2   \big| \ca{P} |^2 \in (0,1)$, and $\textup{TV} \left( \pi_{i,t}^{1}\left(  p_{i} \big|\v{s}_{\hat{\ca{N}}_i}\right) , \pi_{i,t}^{*}\left(  p_{i} |\v{s}\right)  \right) $ denotes the total variation distance between the initial policy $\pi_{i,t}^{1}$ and the optimal policy $\pi_{i,t}^{*}$.
\end{thm}
\vspace{-10pt}
\begin{proof}
	{The proof is presented in Appendix B.}
\end{proof}

\begin{rem}
The theoretical result ensures that the decentralized policy generated by Algorithm \ref{D_algorithm} asymptotically converges to the optimal policy produced by a fully centralized solution.
By choosing a suitable algorithm parameter $\gamma$, the cost gap between the decentralized policy $\pi_{\textup{loc}}$ and the optimal policy $\pi^{*}$ converges to zero at a geometric rate as the number of iterations $R$ increases.
Intuitively, each additional iteration allows the algorithm to refine the localized policies and better approximate the global optimum, thereby improving performance.
However, growing $R$ also raises communication costs among devices, imposing a practical tradeoff between communication overhead and achievable accuracy.
\end{rem}

\section{Simulation Results}\label{sim_result}

In this section, we empirically validate the theoretical findings and assess the effectiveness of the proposed decentralized policy iteration algorithm using real-world datasets.

\subsection{Simulation Setups}

\subsubsection{Learning Model Setting}
We evaluate the proposed algorithm on two standard image classification benchmarks: Fashion-MNIST (FMNIST) \cite{xiao2017fashion} and CIFAR-10 \cite{krizhevsky2009learning}.
A CNN model is employed for the FMNIST dataset, whereas a ResNet-18 backbone~\cite{he2016deep} is utilized for the CIFAR-10 dataset.
The learning rates are set to 0.01 for the FMNIST dataset and 0.05 for the CIFAR-10 dataset, respectively.
Both datasets are divided into $m$ subsets, with each subset assigned to a distinct device as its local training dataset.
We investigate both i.i.d. and non-i.i.d. data distributions.
For i.i.d. scenarios, we implement uniform random data distribution across devices. To emulate realistic heterogeneity in non-i.i.d. settings, we employ a Dirichlet distribution Dir(0.8) for the datasets.

\begin{figure*}[htbp]
	\centering
	\begin{subfigure}{0.49\textwidth}
		\includegraphics[width=1\linewidth]{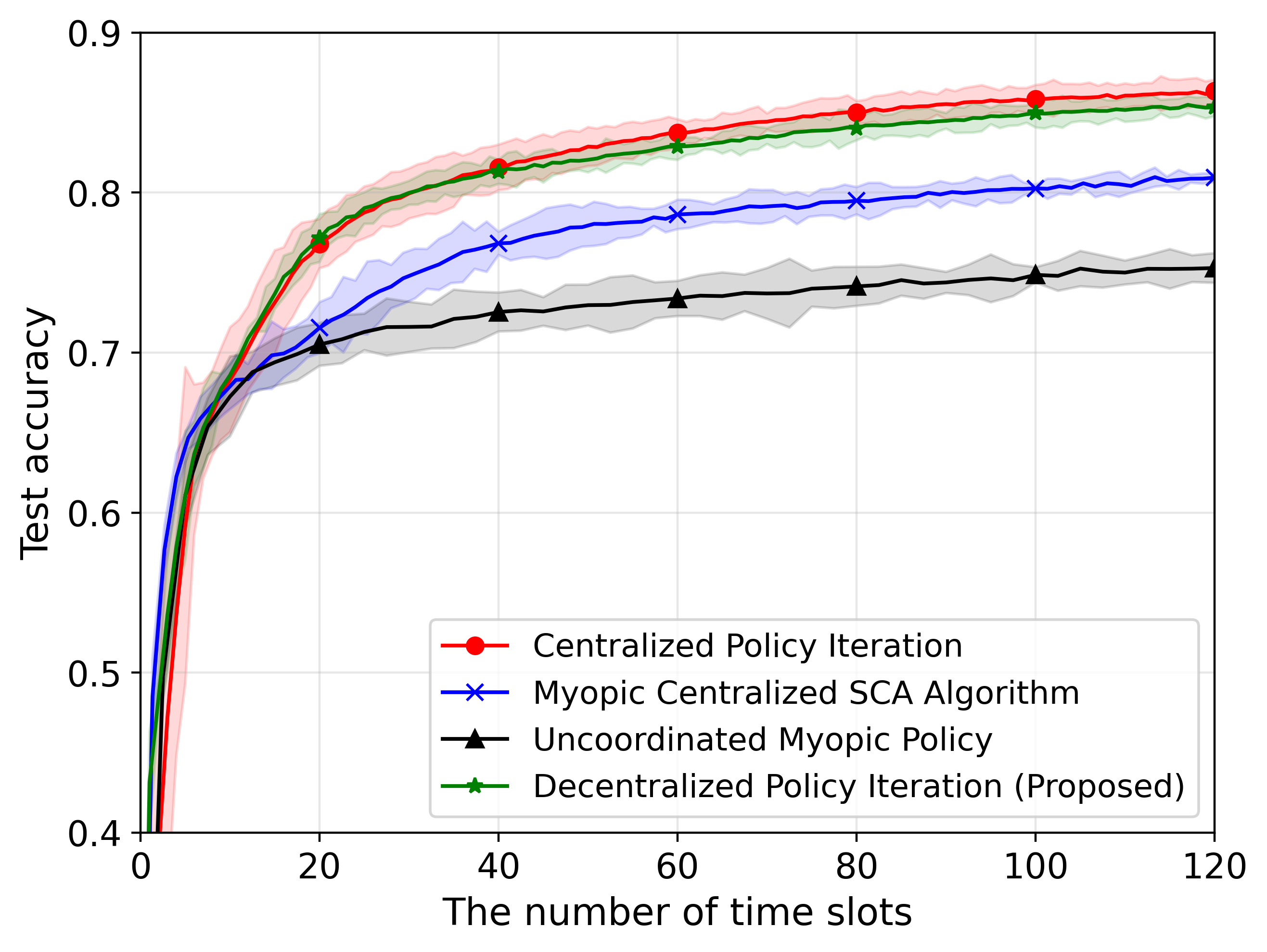}
		\vspace{-17pt}
		\caption{i.i.d.}
		\label{fig:subfig1}
	\end{subfigure}
	\begin{subfigure}{0.49\textwidth}
		\includegraphics[width=1\linewidth]{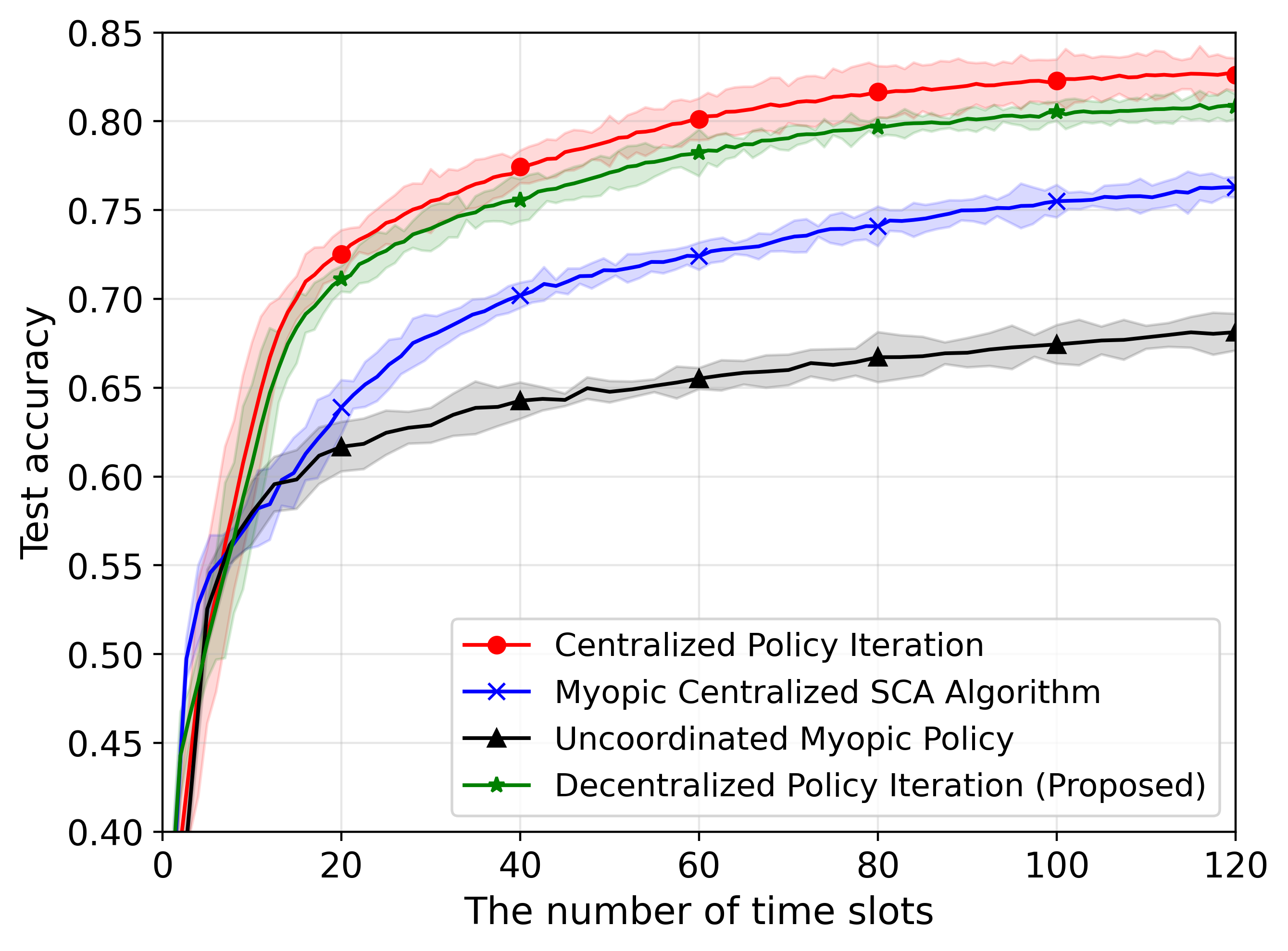}
		\vspace{-17pt}
		\caption{non-i.i.d.}
		\label{fig:subfig2}
	\end{subfigure}
	\vspace{-1pt}
	\caption{Test accuracy v.s. time slots for various transmission schemes on FMNIST.}
	\label{fig:three_subfigs}
\end{figure*}

\subsubsection{Communication Model Setting}
The inter-device wireless channels are modeled under Rayleigh fading and approximated by finite-state Markov channels \cite{tan2000first}.
The channel state at a given time slot is assumed to either remain unchanged or transition to an adjacent state in the following time slot, which aligns with slow-fading conditions. 
Specifically, the channel transition probability from state $H_k$ to state $H_{k^\prime}$ is approximated by \cite{sadeghi2008finite}:
\begin{equation}\label{sim_channel_trans}
	\begin{aligned}
		&\b{P}(h_{i,t+1}=H_{k^\prime}|h_{i,t}=H_{k})\\
		&=\begin{cases}\frac{Z(H_{k+1})\tau}{\b{P}({H_{k}})} , & \text { if } k^\prime = k + 1, \\   \frac{Z(H_{k})\tau}{\b{P}({H_{k}})} , & \text { if } k^\prime = k - 1,\\ 1 -  \frac{Z(H_{k+1})\tau}{\b{P}({H_{k}})} - \frac{Z(H_{k})\tau}{\b{P}({H_{k}})} , & \text { if } k^\prime = k,
		\end{cases}
	\end{aligned}    
\end{equation}
where $Z(H_{k})$ and $\b{P}(H_{k})$ denote the level crossing rate and the steady probability of channel state $H_k$, respectively, and $\tau$ refers to the duration of a single time slot.

\subsubsection{Energy Harvesting Model Setting}
For the energy harvesting model, we adopt a real-world solar irradiance dataset, which captures radiation intensity measurements recorded in June from 2008 to 2010 at a solar observation site located at Elizabeth City State University \cite{Andreas1981}.
Considering the compact size of edge devices, each solar panel is assumed to have an area of 25 $\text{cm}^2$.
The energy conversion efficiency is set to $20\%$, and the battery state of each device is randomly initialized.

\vspace{-2pt}
\subsection{Performance Comparison}


In this section, we present a performance comparison between the proposed decentralized policy iteration algorithm and the following benchmark methods, highlighting its strengths in achieving high learning performance and efficient energy utilization.

\begin{figure*}[htbp]
	\centering
	\begin{subfigure}{0.48\textwidth}
		\includegraphics[width=1\linewidth]{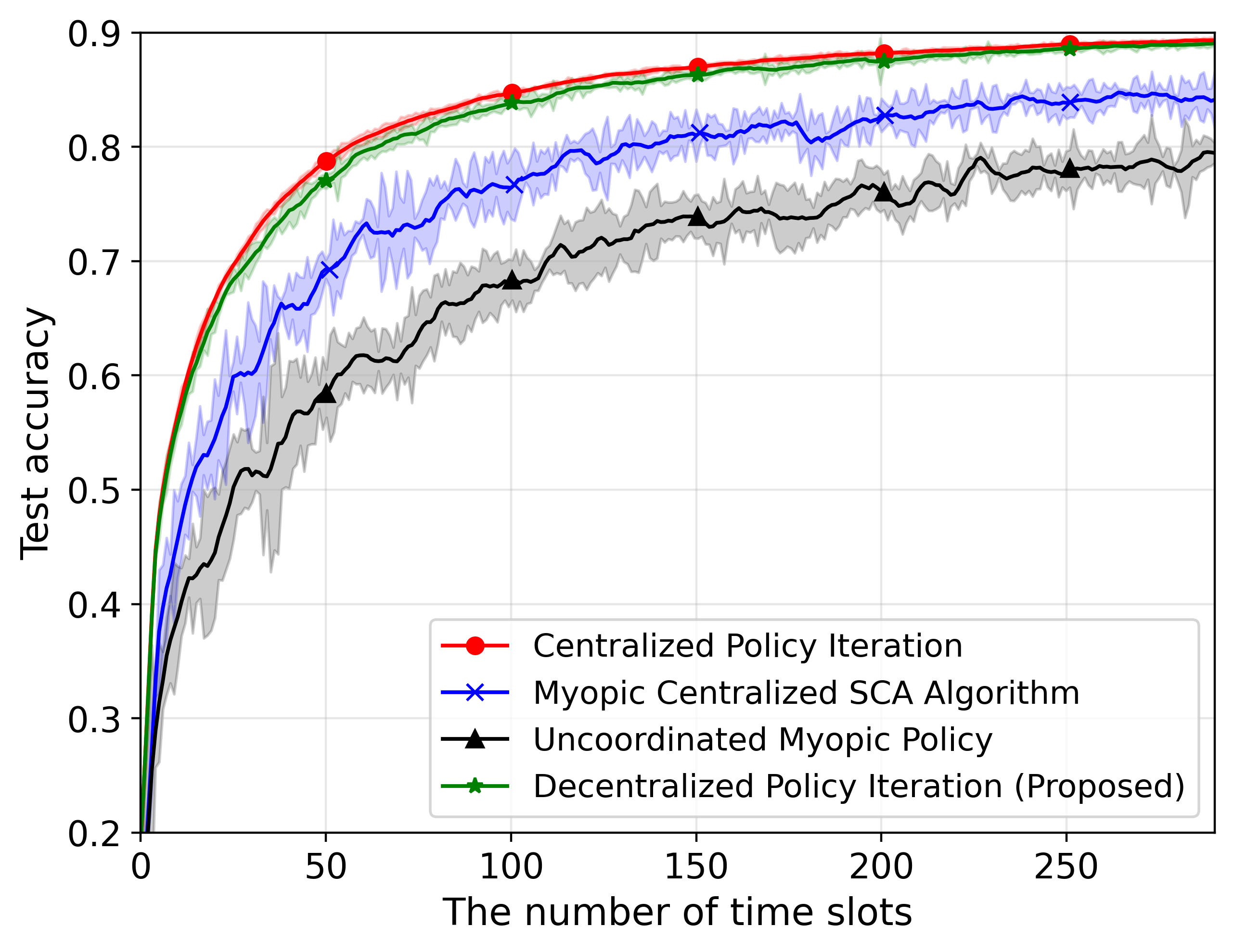}
		\vspace{-17pt}
		\caption{i.i.d.}
		\label{fig:subfig3}
	\end{subfigure}
	\begin{subfigure}{0.48\textwidth}
		\includegraphics[width=1\linewidth]{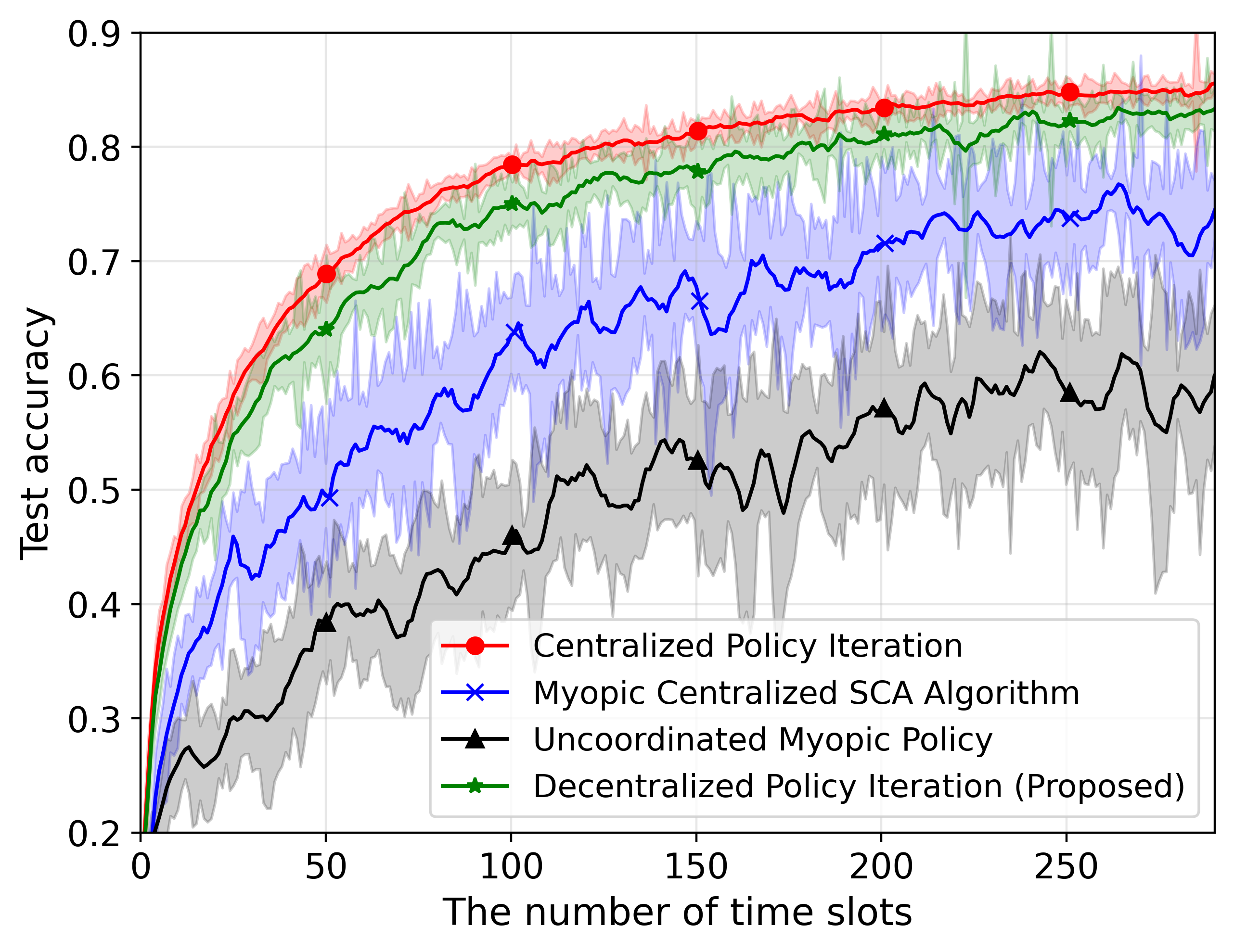}
		\vspace{-17pt}
		\caption{non-i.i.d.}
		\label{fig:subfig4}
	\end{subfigure}
		\vspace{-1pt}
	\caption{Test accuracy v.s. time slots for various transmission schemes on CIFAR-10.}
	\label{fig:three_subfigs5}
\end{figure*}

\begin{itemize}
	\item \textbf{Centralized Policy Iteration:} In this benchmark, all devices transmit their full state information to a single central node, which then executes a classical policy iteration procedure on the global system state. While this approach can theoretically achieve the optimal performance, it is impractical for EH-enabled DFL systems due to the prohibitive communication overhead and computational complexity.
	
	\item \textbf{Myopic Centralized SCA Algorithm:} In this benchmark, a centralized node minimizes the instantaneous system cost at each time slot without accounting for future network states or energy evolutions. Specifically, it collects the current state information from all devices and applies a successive convex approximation (SCA) approach to address the following optimization problem. 
		\begin{align}
			\underset{\v{p}_t}{\text{min}} & \hspace{-1pt} \sum_{j=1}^{m}\hspace{-1pt} \sum_{i\neq j}^{m} a_{i,j} \!\! \left( \!\hspace{-1pt}1\!-\!\exp\! \!\left(\!\hspace{-1pt}-\frac{ \varphi \! \left(\! \sum_{\hspace{-1pt}j^{\hspace{-0.5pt} \prime}\hspace{-1pt} \in \ca{N}_i \hspace{-1pt}\backslash j } \hspace{-1pt} p_{j^{\hspace{-0.5pt}\prime}\hspace{-1pt},t} h_{i,j^{\hspace{-0.5pt} \prime} \hspace{-1pt},t} \!+\!  \sigma^2_{i} \right) }{p_{j,t} h_{i,j,t}}\!\right)\!\! \right) \hspace{-1pt} \nonumber \\
			\text{s.t.}   &~ \textup{C1}.
		\end{align}

	\item \textbf{Uncoordinated Myopic Policy:} In this decentralized baseline, each device simply expends all of its available energy to maximize immediate transmission power, disregarding both future energy needs and interference management. Although straightforward to implement, this policy often leads to suboptimal global performance and frequent energy depletion, since no inter-device coordination or long-term planning is considered.
	
\end{itemize}

\begin{figure*}[htbp]
	\vspace{-3pt}
	\centering
	\begin{subfigure}{0.48\textwidth}
		\includegraphics[width=1\linewidth]{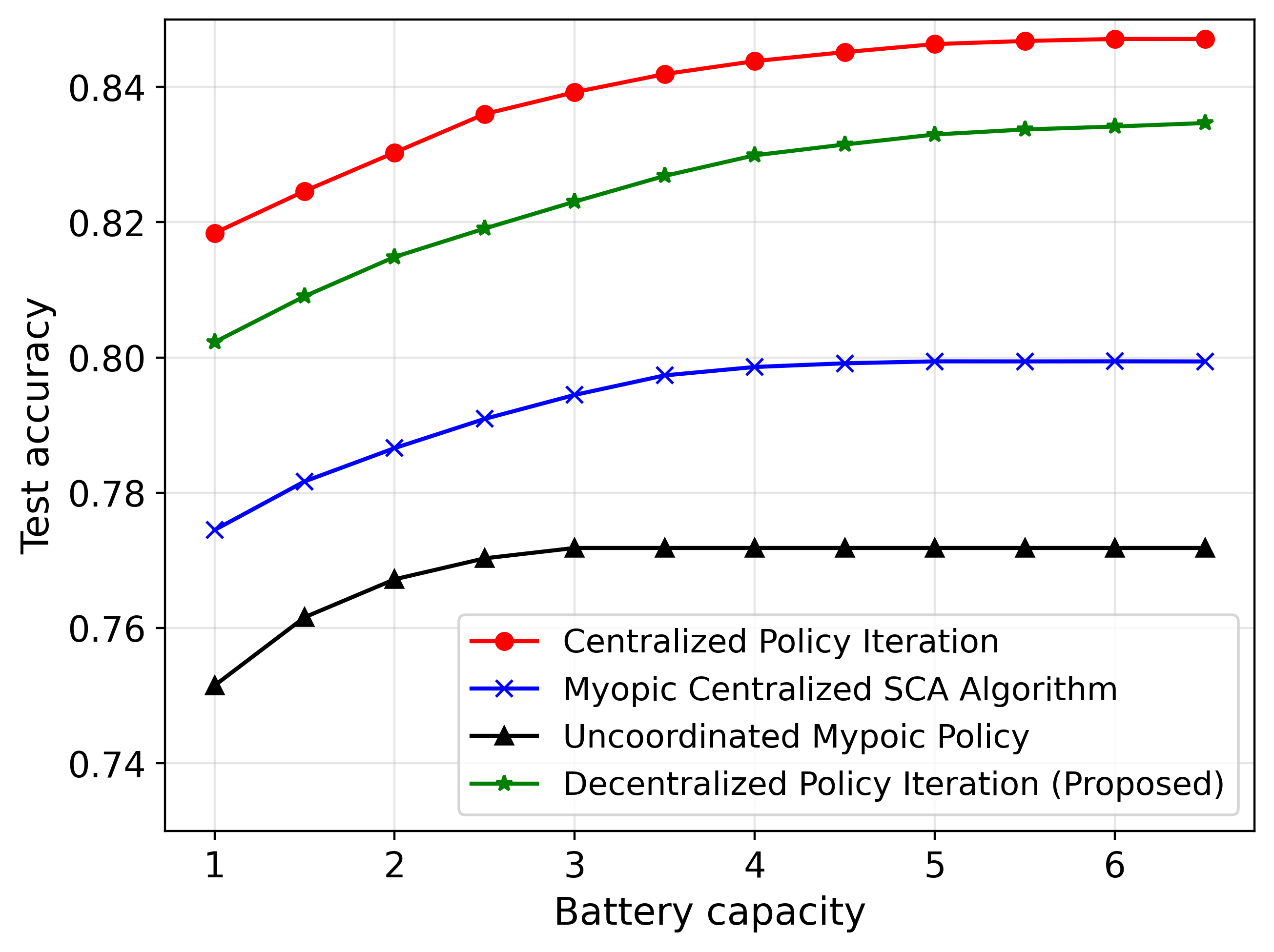}
		\vspace{-17pt}
		\caption{i.i.d.}
		\label{fig:subfig6}
	\end{subfigure}
	\begin{subfigure}{0.48\textwidth}
		\includegraphics[width=1\linewidth]{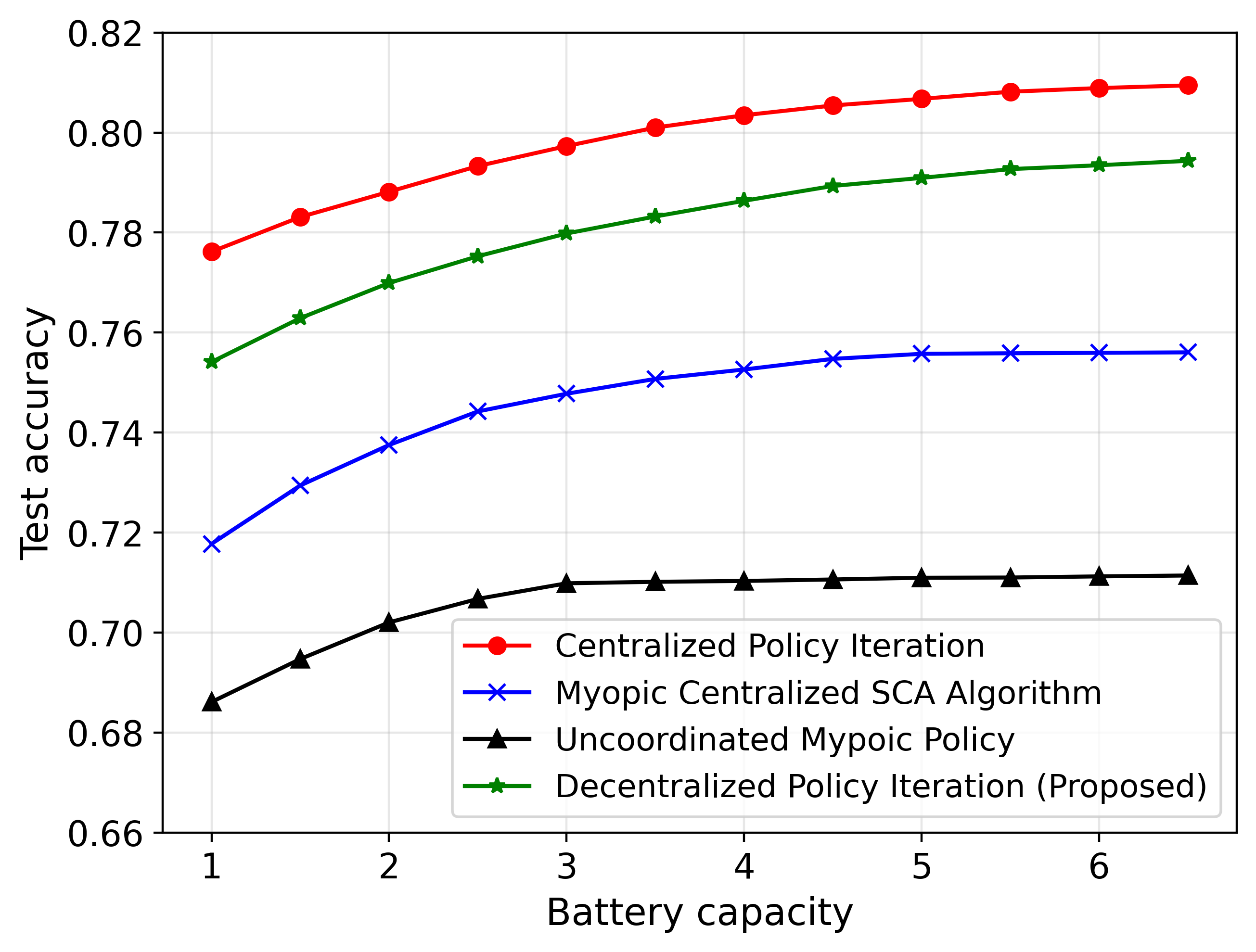}
		\vspace{-17pt}
		\caption{non-i.i.d.}
		\label{fig:subfig7}
	\end{subfigure}
	\vspace{-4pt}
	\caption{Test accuracy v.s. battery capacity for different transmission strategies on FMNIST.}
	\label{fig:three_subfigs8}
	\vspace{-2pt}
\end{figure*}

Fig. \ref{fig:three_subfigs} depicts the learning performance of four transmission strategies on the FMNIST dataset, under both i.i.d. and non-i.i.d. settings.
As observed, the proposed decentralized policy iteration algorithm closely matches the performance of the centralized policy iteration.
In contrast, the myopic centralized SCA algorithm forgoes long-term considerations by focusing only on immediate, per-slot optimization, resulting in lower accuracy. Meanwhile, the uncoordinated myopic policy performs the worst overall, as each device transmits at maximal power without regard for future battery levels or interference constraints.
A key observation is that the proposed decentralized policy iteration algorithm requires only two-hop local information and modest communication, yet still converges at a rate similar to that of the global benchmark. This provides a compelling balance between practicality and performance, particularly in EH-enabled DFL systems where centralized solutions are infeasible.

We further compare the four transmission schemes on the CIFAR-10 dataset, and the corresponding results are presented in Fig.~\ref{fig:three_subfigs5}.
The centralized policy iteration algorithm maintains the highest test accuracy, while the proposed decentralized policy iteration algorithm again achieves nearly the optimal performance while only using local two-hop information.
By contrast, the myopic centralized SCA algorithm attains moderate accuracy due to its short-term focus, and the uncoordinated myopic policy yields the lowest accuracy because of unsustainably high power usage. These results reinforce that the proposed approach strikes a practical balance between coordination overhead and near-optimal accuracy, even when scaling to a more complex dataset such as the CIFAR-10 dataset.

We also investigate the effect of battery capacity on the learning performance, as illustrated in Fig. \ref{fig:three_subfigs8}. As the battery capacity increases, all compared methods benefit from the additional energy available for data transmission, thereby improving classification accuracy. Nonetheless, the proposed decentralized policy iteration consistently outperforms the myopic benchmarks and performs near the optimal reference provided by the centralized policy iteration.
Moreover, as battery capacity increases, the performance margin between the proposed decentralized approach and the two myopic baselines widens.
This result indicates that the proposed method’s long-term planning allows it to exploit larger energy reserves more effectively, whereas the myopic strategies, which emphasize short-term objectives, cannot fully harness the additional battery capacity.

{\color{black}

\subsection{Ablation Studies}

In this subsection, we first conduct an ablation study on the neighborhood scope used for localized policy optimization. Specifically, we compare four variants of the proposed algorithm with different neighborhood scopes: no communication (each device acts based only on its own local state), one-hop, two-hop (proposed algorithm), and three-hop communication. In these variants, we generalize the two-hop localized policy by replacing the default two-hop neighborhood with a $k$-hop neighborhood (with $k\in\{0,1,2,3\}$). Table~\ref{tab:neighborhood_ablation} presents the final test accuracy under both i.i.d. and non-i.i.d. data distributions. As observed, the no-communication and one-hop variants yield clearly inferior performance, indicating that overly limited neighborhood information is insufficient for effective decentralized coordination. The proposed two-hop design achieves the best accuracy in both settings. Extending the neighborhood to three hops does not provide further improvement and can even degrade accuracy, which is consistent with the substantially larger localized state--action space and the increased coordination overhead induced by larger neighborhoods.

\begin{table}[t]
	\color{black}
	\centering
	\normalsize
	\caption{\color{black}Test accuracy of the proposed algorithm under different communication ranges and data distributions.}
	\label{tab:neighborhood_ablation}
	\begin{tabular}{p{4.3cm} p{1.6cm} p{1.6cm}}
		\toprule
		\textbf{Communication Range} & \textbf{IID} & \textbf{Non-IID} \\
		\midrule
		No communication     &   75.9\%  &   65.2\% \\
		One-hop              &   73.3\%  &   62.2\%  \\
		\textbf{Two-hop (proposed)}   &   \textbf{84.2\%}  &   \textbf{79.9\%}  \\
		Three-hop            &   80.9\%  &   71.7\%  \\
		\bottomrule
	\end{tabular}
\end{table}

We further evaluate the tradeoff between learning performance and the communication overhead induced by the number of policy-improvement iterations $R$. Fig.~\ref{fig:acc_r} shows the final test accuracy as a function of $R$, while keeping all other settings fixed. As $R$ increases, the test accuracy improves and then gradually saturates, demonstrating diminishing returns from further increasing $R$ and empirically confirming the theoretical tradeoff between communication cost and learning performance. In particular, when $R=0$, the policy-improvement step is skipped and no exchange and aggregation of localized Q-values is performed, leading to markedly worse and less stable performance. Based on this result, we set $R=20$ in all simulations to balance learning performance and the communication overhead induced by the policy-improvement iterations.

\begin{figure}[t]
	\centering
	\includegraphics[width=0.49\textwidth]{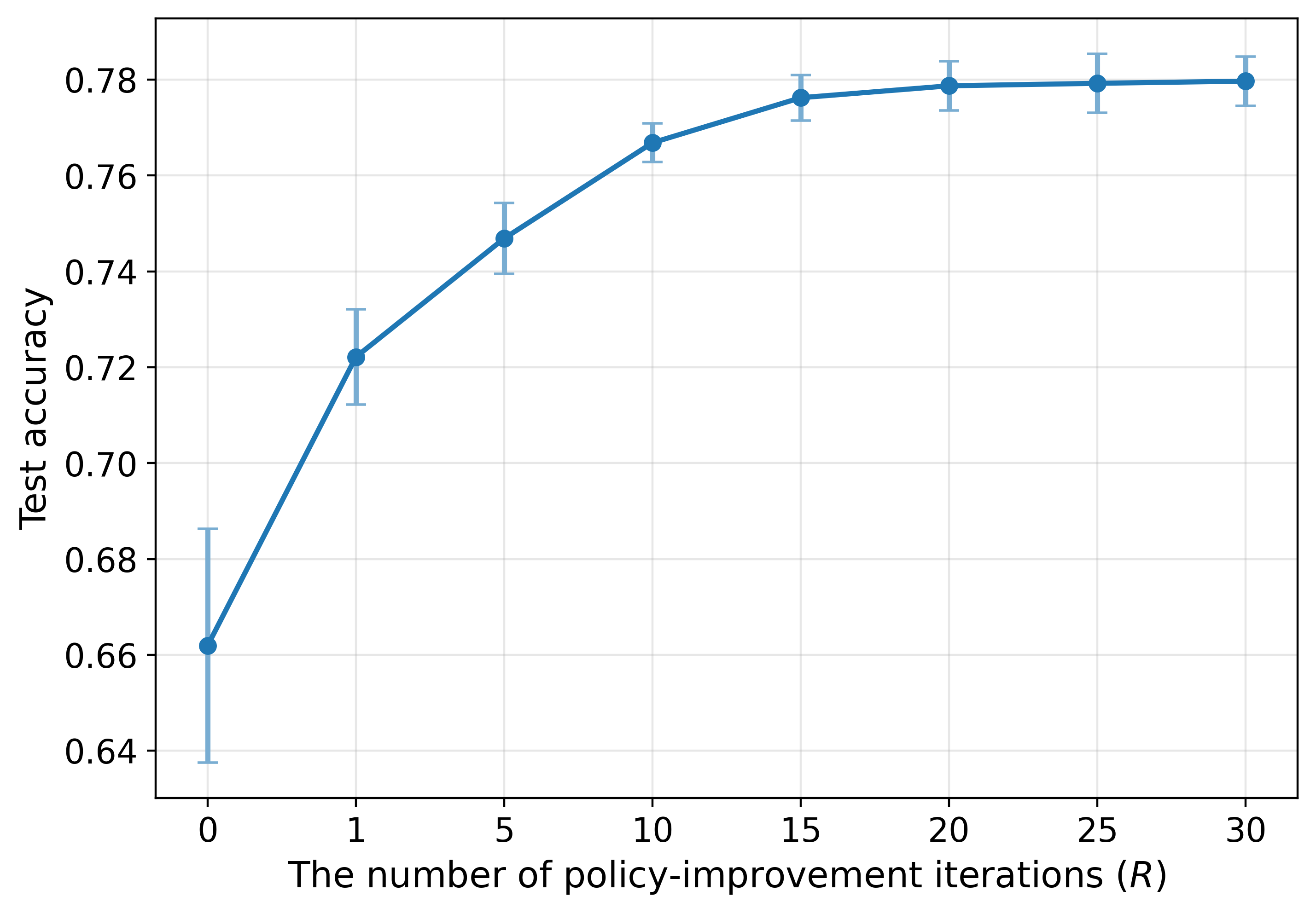}
	\caption{\color{black}Final test accuracy versus the number of policy-improvement iterations $R$.}
	\label{fig:acc_r}
\end{figure}

}

\section{Conclusion}\label{Sec_conclusion}
 
In this paper, we have investigated wireless DFL with EH devices for sustainable model training among energy-constrained edge devices. We have derived a convergence bound for EH-enabled DFL, showing how partial device participation and transmission packet drops impact learning performance. Building upon the convergence bound, we have investigated a joint device scheduling and power control problem and modeled it as a multi-agent MDP.
To tackle the prohibitive complexity and communication overhead of traditional centralized MDP methods, we have proposed a fully decentralized policy iteration algorithm that leverages only local state information from two-hop neighboring devices.
We have further provided a theoretical analysis showing the asymptotic optimality of the proposed decentralized solution.
Finally, we have presented comprehensive numerical experiments on real-world datasets, demonstrating the effectiveness of the proposed algorithm.

\begingroup        

\setlength{\abovedisplayskip}{1.1pt}
\setlength{\belowdisplayskip}{1.1pt}
\setlength{\abovedisplayshortskip}{2pt}
\setlength{\belowdisplayshortskip}{2pt}
\setlength{\jot}{2pt}

\section*{Appendix}

\subsection{Proof of Theorem \ref{thm_convergence_bound}}
To facilitate convergence analysis, we first introduce two auxiliary variables $\v{\bar w}_{t+\frac{1}{2}} = \frac{1}{m} \sum_{i=1}^{m}  \v{w}_{i,t+\frac{1}{2}} $ and $\v{\bar w}_{t}^{(K)} = \frac{1}{m} \sum_{i=1}^{m}\v{w}_{i,t}^{(K)}$.
According to the model update rules in equations \eqref{eq:local_update}, \eqref{eq:update_scheduling}, \eqref{update_global}, it follows that
	\begin{align}
		& \v{\bar w}_{t+1} - \v{\bar w}_{t} \nonumber\\
		\aeq &\v{\bar w}_{t+\frac{1}{2}} - \v{\bar w}_{t} + \frac{\eta}{m} \sum_{j=1}^{m} \sum_{i\neq j}^{m} a_{i,j} \left( 1  -  \zeta_{i,j,t} \right) \beta_{j,t} \sum_{k=0}^{K-1} {\v{g}}_{j,t,k} \nonumber\\
		\beq & \v{\bar w}_{t}^{(K)} - \v{\bar w}_{t} + \frac{\eta}{m} \sum_{j=1}^{m}\left(  1 -\beta_{j,t} \right)  \sum_{k=0}^{K-1} {\v{g}}_{j,t,k}  \nonumber\\
		&+ \frac{\eta}{m} \sum_{j=1}^{m} \sum_{i\neq j}^{m} a_{i,j} \left( 1  -  \zeta_{i,j,t} \right) \beta_{j,t} \sum_{k=0}^{K-1} {\v{g}}_{j,t,k} ,
	\end{align}
where (a) holds since the mixing matrix is doubly stochastic, and (b) follows from \eqref{eq:update_scheduling}. Then, according to Assmuption \ref{ass_smoothness_assump}, we have
	\begin{align}\label{app_1_1}
		& \b{E}\left[  F\left( \bar{\v{w}}_{t+1} \right)\right] \nonumber \\ 
		\leq &  \b{E} \left[  F\left( \bar{\v{w}}_{t} \right) \right]   + \b{E}\left[  \nabla F\left( \bar{\v{w}}_{t} \right)^{\s{T}} \left( \v{\bar w}_{t+1} - \v{\bar w}_{t}  \right)  \right]  \nonumber \\
		&+ \frac{L}{2} \b{E}\left[  \left\| \v{\bar w}_{t+1} - \v{\bar w}_{t}  \right\|^2 \right]  \nonumber  \\
		= &\b{E} \left[  F\left( \bar{\v{w}}_{t} \right) \right]   + \b{E}\left[  \nabla F\left( \bar{\v{w}}_{t} \right)^{\s{T}} \left(  \v{\bar w}_{t}^{(K)}  - \v{\bar w}_{t}  \right)  \right]  \nonumber  \\
		& + \frac{\eta}{m} \b{E}\Biggl[  \nabla F\left( \bar{\v{w}}_{t} \right)^{\s{T}} \Biggl( \sum_{j=1}^{m}\left(  1 -\beta_{j,t}\right)  \sum_{k=0}^{K-1} {\v{g}}_{j,t,k}  \nonumber \\
		&+  \sum_{j=1}^{m} \sum_{i\neq j}^{m} a_{i,j} \left( 1  -  \zeta_{i,j,t} \right)  \sum_{k=0}^{K-1} {\v{g}}_{j,t,k} \Biggr)  \Biggr]  \nonumber  \\
		& + \frac{L}{2} \b{E}\Biggl[  \Biggl\|  \v{\bar w}_{t}^{(K)}  - \v{\bar w}_{t} + \frac{\eta}{m}\sum_{j=1}^{m}\left( 1 - \beta_{j,t}\right)  \sum_{k=0}^{K-1} {\v{g}}_{j,t,k} \nonumber  \\
		&+ \frac{\eta}{m} \sum_{j=1}^{m} \sum_{i\neq j}^{m} a_{i,j} \left( 1  -  \zeta_{i,j,t} \right)  \sum_{k=0}^{K-1} {\v{g}}_{j,t,k} \Biggr\|^2 \Biggr] \nonumber  \\
		\aleq & \b{E} \left[  F\left( \bar{\v{w}}_{t} \right) \right]   + \b{E}\left[  \nabla F\left( \bar{\v{w}}_{t} \right)^{\s{T}} \left(  \v{\bar w}_{t}^{(K)}  - \v{\bar w}_{t}  \right)  \right] \nonumber \\
		&+ \eta \sqrt{K} G^2 \b{E}\left[ \sum_{j=1}^{m}\left( 1 - \beta_{j,t} \right) + \sum_{j=1}^{m} \sum_{i\neq j}^{m} a_{i,j} \left( 1  -  \zeta_{i,j,t} \right) \beta_{j,t}   \right]  \nonumber  \\
		& + L \b{E}\left[  \left\|  \v{\bar w}_{t}^{(K)}  - \v{\bar w}_{t}\right\|^2 \right]  \!+\! \frac{\eta L}{m^2} \b{E}\Biggl[  \Biggl\| \sum_{j=1}^{m}\left( 1 - \beta_{j,t} \right)\!  \sum_{k=0}^{K-1} {\v{g}}_{j,t,k}  \nonumber \\
		&+ \sum_{j=1}^{m} \sum_{i\neq j}^{m} a_{i,j} \left( 1  -  \zeta_{i,j,t} \right)  \sum_{k=0}^{K-1} {\v{g}}_{j,t,k}  \Biggr\|^2 \Biggr] \nonumber  \\
		\bleq &\b{E} \left[  F\left( \bar{\v{w}}_{t} \right) \right]   + \b{E}\left[  \nabla F\left( \bar{\v{w}}_{t} \right)^{\s{T}} \left(  \v{\bar w}_{t}^{(K)}  - \v{\bar w}_{t}  \right)  \right] \nonumber \\
		&+ \eta \sqrt{K} G^2 \b{E}\left[ \sum_{j=1}^{m}\left( 1 - \beta_{j,t} \right) + \sum_{j=1}^{m} \sum_{i\neq j}^{m} a_{i,j} \left( 1  -  \zeta_{i,j,t} \right) \beta_{j,t}   \right]  \nonumber  \\
		& + L \b{E}\left[  \left\|  \v{\bar w}_{t}^{(K)}  - \v{\bar w}_{t}\right\|^2 \right]  \nonumber \\
		& + \eta K L G^2 \b{E}\left[ \sum_{j=1}^{m}\left( 1 - \beta_{j,t}\right) + \sum_{j=1}^{m} \sum_{i\neq j}^{m} a_{i,j} \left( 1  -  \zeta_{i,j,t} \right) \beta_{j,t}   \right]   \nonumber \\
		= &  \b{E} \left[  F\left( \bar{\v{w}}_{t} \right) \right]   + \b{E}\left[  \nabla F\left( \bar{\v{w}}_{t} \right)^{\s{T}} \left(  \v{\bar w}_{t}^{(K)}  - \v{\bar w}_{t}  \right)  \right] \nonumber\\
		&+ L \b{E}\left[  \left\|  \v{\bar w}_{t}^{(K)} - \v{\bar w}_{t}\right\|^2 \right] + \eta \left(  K L + \sqrt{K} \right)  G^2  \nonumber \\
		&\times \b{E}\left[ \sum_{j=1}^{m}\left( 1 - \beta_{j,t} \right) + \sum_{j=1}^{m} \sum_{ i\neq j}^{m}  a_{i,j} \left( 1  -  \zeta_{i,j,t} \right) \beta_{j,t}   \right] , 
	\end{align}
where (a) and (b) follow from Assumption \ref{ass_bounded_G_assump}.

Then, according to the proof of Theorem 1 in \cite{sun2021decentralized}, when Assumptions \ref{ass_smoothness_assump} and \ref{ass_bounded_var_assump} hold, $  \b{E}\left[  \nabla F\left( \bar{\v{w}}_{t} \right)^{\s{T}} \left(  \v{\bar w}_{t}^{(K)}  - \v{\bar w}_{t}  \right)  \right]$ and $L \b{E}\left[  \left\|  \v{\bar w}_{t}^{(K)} - \v{\bar w}_{t}\right\|^2 \right]$ are respectively upper bounded by
	\begin{align}\label{app_1_2}
		&  \b{E}\left[  \nabla F\left( \bar{\v{w}}_{t} \right)^{\s{T}} \left(  \v{\bar w}_{t}^{(K)}  - \v{\bar w}_{t}  \right)  \right] \nonumber\\
		\leq & -\frac{\eta K}{2} \b{E} \left[  \left\| \nabla F\left( \bar{\v{w}}_{t} \right) \right\|^2  \right] + \frac{\eta^3 L^2 K^2 \left( 8 \sigma_l^2 + 32 K \sigma_g^2 \right) }{2} \nonumber\\
		&+  \frac{16\eta^3 L^2 K^3}{m} \sum_{i=1}^{m} \b{E} \left[  \left\| \nabla F\left( {\v{w}}_{i,t} \right) \right\|^2\right] ,
	\end{align}
and
\begin{equation}
	\begin{aligned}\label{app_1_3}
		&  L \b{E}\left[  \left\|  \v{\bar w}_{t}^{(K)}  - \v{\bar w}_{t}\right\|^2 \right] \\
		\leq & \eta^2 LK\left( 8 \sigma_l^2 + 32 K \sigma_g^2 \right) \!+\! \frac{16 \eta^2 L K^2}{m} \sum_{i=1}^{m} \b{E} \left[  \left\| \nabla F\left( {\v{w}}_{i,t} \right) \right\|^2 \right] .
	\end{aligned}
\end{equation}

Plugging \eqref{app_1_2} and \eqref{app_1_3} into \eqref{app_1_1}, we have
	\begin{align}\label{app_1_4}
		& \b{E}\left[  F\left( \bar{\v{w}}_{t+1} \right)\right] \nonumber \\ 
		\leq &  \b{E} \!\left[  F\left( \bar{\v{w}}_{t} \right) \right] \!-\!\frac{\eta K}{2} \b{E} \!\left[  \left\| \nabla F\left( \bar{\v{w}}_{t} \right) \right\|^2  \right] \!+\! \frac{\eta^3 L^2 K^2 \left( 8 \sigma_l^2 \!+\! 32 K \sigma_g^2 \right) }{2} \nonumber\\
		&+ \frac{16\eta^3 L^2 K^3}{m} \sum_{i=1}^{m} \b{E} \left[  \left\| \nabla F\left( {\v{w}}_{i,t} \right) \right\|^2\right]\nonumber \\
		&+  \eta^2 LK\left( 8 \sigma_l^2 \!+\! 32 K \sigma_g^2 \right) \!+\! \frac{32 \eta^2 L K^2}{m} \sum_{i=1}^{m} \b{E}\! \left[  \left\| \nabla F\left( {\v{w}}_{i,t} \right) \right\|^2 \right]\nonumber\\
		& + \frac{\eta \left(  K L + \sqrt{K} \right) G^2}{m-1}\nonumber\\
		& \times  \b{E}\left[ \sum_{j=1}^{m} \sum_{i\neq j}^{m} \left(  (m - 1) a_{i,j} \left( 1  -  \zeta_{i,j,t} \right) - 1\right)  \beta_{j,t} + 1 \right] \nonumber\\
		= &  \b{E} \left[  F\left( \bar{\v{w}}_{t} \right) \right] -\frac{\eta K}{2} \b{E} \left[  \left\| \nabla F\left( \bar{\v{w}}_{t} \right) \right\|^2  \right] \nonumber\\
		& + \frac{\left( \eta^3 L^2 K^2 + 2\eta^2 L K\right)  \left( 8 \sigma_l^2 + 32 K \sigma_g^2 \right) }{2}\nonumber\\
		&+ \frac{16 \eta^2 L K^2 + 16 \eta^3 L^2 K^3}{m} \sum_{i=1}^{m} \b{E} \left[  \left\| \nabla F\left( {\v{w}}_{i,t} \right) \right\|^2 \right]\nonumber\\
		& +\! \frac{\eta \!\left(  K L \!+\! \sqrt{K} \right) \!G^2}{m-1} \!   \sum_{j=1}^{m} \sum_{ i\neq j}^{m} (\left(  (m \!-\! 1) a_{i,j} q_{i,j,t}\!-\! 1\right)  \beta_{j,t} \!+\! 1 ). \nonumber\\
	\end{align}
Furthermore, we have
	\begin{align}\label{app_1_5}
		& \frac{\sum_{i=1}^{m} \b{E} \left[  \left\| \nabla F\left( {\v{w}}_{i,t} \right) \right\|^2 \right] }{m} \nonumber\\
		= &  \frac{\sum_{i=1}^m \mathbb{E}\left[ \left\| \nabla F\left( {\v{w}}_{i,t} \right) - \nabla F\left( {\v{\bar w}}_{t} \right) + \nabla F\left( {\v{\bar w}}_{t} \right) \right\|^2 \right] }{m} \nonumber\\
		\leq &\frac{2\sum_{i=1}^m \mathbb{E}\left[ \left\| \nabla F\left( {\v{w}}_{i,t} \right) - \nabla F\left( {\v{\bar w}}_{t} \right)  \right\|^2 \right] }{m} \nonumber\\
		& + \frac{2 \sum_{i=1}^m \mathbb{E}\left[ \left\| \nabla F\left( {\v{\bar w}}_{t} \right) \right\|^2 \right] }{m} \nonumber\\
		\aleq & \frac{2L^2 \sum_{i=1}^m \| {\v{w}}_{i,t} - {\v{\bar w}}_{t} \|^2}{m} 
		+ 2 \mathbb{E}\left[ \left\| \nabla F\left( {\v{\bar w}}_{t} \right) \right\|^2 \right] \nonumber\\
		\bleq &\frac{2 \eta^2 L^2  \left( 8 K \sigma_l^2 \!+\! 32 K^2 \sigma_g^2 \!+\! 32 K^2 G^2 \right) }{1 - \lambda} \!+\! 2 \mathbb{E}\left[ \left\| \nabla F\left( {\v{\bar w}}_{t} \right) \right\|^2 \right],
	\end{align}
where (a) holds due to Assumption \ref{ass_smoothness_assump}; (b) follows from Lemma 4 in \cite{sun2021decentralized} and we omit it for brevity.

Combining \eqref{app_1_5} into \eqref{app_1_4}, we can obtain that
	\begin{align}\label{app_1_6}
		& \b{E}\left[  F\left( \bar{\v{w}}_{t+1} \right)\right] \nonumber \\
		\leq &  \b{E} \left[  F\left( \bar{\v{w}}_{t} \right) \right] -\frac{\eta K}{2} \b{E} \left[  \left\| \nabla F\left( \bar{\v{w}}_{t} \right) \right\|^2  \right]  \nonumber\\
		&+ \frac{\left( \eta^3 L^2 K^2 + 2\eta^2 L K\right)  \left( 8 \sigma_l^2 + 32 K \sigma_g^2 \right) }{2} \nonumber\\
		&+ \frac{\left( 64 \eta^4 L^3 K^2 \!+ \!32 \eta^5 L^4 K^3\right) \left( 8 K \sigma_l^2 \!+\! 32 K^2 \sigma_g^2 \!+\! 32 K^2 G^2 \right)  }{m(1 - \lambda)}   \nonumber\\
		&+ \frac{32 \eta^2 L K^2 + 32 \eta^3 L^2 K^3}{m}  \mathbb{E}\left[ \left\| \nabla F\left( {\v{\bar w}}_{t} \right) \right\|^2 \right]  \nonumber\\
		& +\! \frac{\eta \!\left(  K L \!+\! \sqrt{K} \right) \!G^2}{m-1} \!   \sum_{j=1}^{m} \sum_{ i\neq j}^{m} ( \left(  (m \!-\! 1) a_{i,j} q_{i,j,t}\!-\! 1\right)  \beta_{j,t} \!+\! 1 )  \nonumber\\
		= &  \b{E} \! \left[  F\!\left( \bar{\v{w}}_{t} \right) \right]\!+\! \left( \! \frac{32 \eta^2 L K^2 \!+\! 32 \eta^3 L^2 K^3}{m} \!-\!\frac{\eta K}{2} \!\right) \! \b{E}\! \left[ \! \left\| \nabla F\!\left( \bar{\v{w}}_{t}  \right) \right\|^2 \!\right]  \nonumber\\
		&+ \frac{\left( \eta^3 L^2 K^2 + 2\eta^2 L K\right)  \left( 8 \sigma_l^2 + 32 K \sigma_g^2 \right) }{2} \nonumber\\
		&+ \frac{\left( 64 \eta^4 L^3 K^2 \!+ \!32 \eta^5 L^4 K^3\right) \left( 8 K \sigma_l^2 \!+\! 32 K^2 \sigma_g^2 \!+\! 32 K^2 G^2 \right)  }{m(1 - \lambda)}  \nonumber\\
		& +\! \frac{\eta \!\left(  K L \!+\! \sqrt{K} \right) \!G^2}{m-1} \!   \sum_{j=1}^{m} \sum_{ i\neq j}^{m} ( \left(  (m \!-\! 1) a_{i,j} q_{i,j,t}\!-\! 1\right)  \beta_{j,t} \!+\! 1 ) .  \nonumber\\
	\end{align}
Summing the inequality \eqref{app_1_6} from $t=1$ to $T$, it follows that
	\begin{align}
		& \frac{1}{T} \left( \frac{\eta K}{2} - \frac{32 \eta^2 L K^2 + 32 \eta^3 L^2 K^3}{m}  \right)  \sum_{t=1}^{T}\b{E} \left[ \left\| \nabla F\left( \bar{\v{w}}_{t} \right) \right\|^2 \right]\nonumber\\
		\leq & \frac{ \b{E} \!\left[ F\!\left(\hspace{-1pt} \bar{\v{w}}_{1} \hspace{-1pt}\right) \!-\! F\!\left(\hspace{-1pt} \v{\bar w}_{T+1}\hspace{-1pt}\right) \right] }{T} \!+\! \frac{\left( \eta^3 \hspace{-1pt} L^2\hspace{-1pt} K^2 \!+\! 2\eta^2 \hspace{-1pt} L K\right) \! \left( 8 \sigma_l^2 \!+\! 32 K \sigma_g^2 \right) }{2}\nonumber\\
		&\!\!\!\! + \frac{\left( 64 \eta^4 L^3 K^2 \!+ \!32 \eta^5 L^4 K^3\right) \left( 8 K \sigma_l^2 \!+\! 32 K^2 \sigma_g^2 \!+\! 32 K^2 G^2 \right)  }{m(1 - \lambda)}  \nonumber\\
		& \!\!\!\! +\! \frac{\eta \!\left( \! K L \!+\! \sqrt{K} \right) \!G^2}{(m-1)T} \! \sum_{t=1}^{T} \! \sum_{j=1}^{m}\! \sum_{ i\neq j}^{m}\! ( \left(  (m \!-\! 1) a_{i,j} q_{i,j,t}\!-\! 1\right)  \beta_{j,t} \!+\! 1 ).
	\end{align}
Then, we can further rearrange terms and obtain that
\begin{equation}
	\begin{aligned}
		& \frac{1}{T} \left( \frac{1}{2} - \frac{32 \eta L K + 32 \eta^2 L^2 K^2}{m}  \right)  \sum_{t=1}^{T}\b{E} \left[  \left\| \nabla F\left( \bar{\v{w}}_{t} \right) \right\|^2 \right] \\
		\leq & \frac{  \b{E}\left[  F\!\left( \bar{\v{w}}_{1} \right) \!-\! F\!\left( \v{\bar w}_{T+1}\right) \right] }{\eta K T} \!+\! \frac{\left( \eta^2 L^2 K \!+\! 2\eta L\right)  \left( 8 \sigma_l^2 \!+\! 32 K \sigma_g^2 \right) }{2}\\
		&+ \frac{\left( 64 \eta^3 L^3 K \!+ \!32 \eta^4 L^4 K^2\right) \left( 8 K \sigma_l^2 \!+\! 32 K^2 \sigma_g^2 \!+\! 32 K^2 G^2 \right)  }{m(1 - \lambda)}  \\
		& +\! \frac{ \!\left( \! K L \!+\! \sqrt{K} \right) \!G^2}{ (m-1)K T} \! \sum_{t=1}^{T} \! \sum_{j=1}^{m}\! \sum_{ i\neq j}^{m}\! ( \left(  (m \!-\! 1) a_{i,j} q_{i,j,t}\!-\! 1\right)  \beta_{j,t} \!+\! 1 ).\\
	\end{aligned}
\end{equation}

By taking $\eta =  \frac{\sqrt{m}}{64 LK \sqrt{T}}$, we have
	\begin{align}
		& \frac{1}{T}  \sum_{t=1}^{T} \b{E} \left[  \left\| \nabla F\left( \bar{\v{w}}_{t} \right) \right\|^2 \right]  \nonumber\\
		\leq & \frac{ 4 \b{E} \left[ F\!\left( \bar{\v{w}}_{1} \right) \!-\! F\!\left( \v{\bar w}_{T+1}\right) \right] }{\eta K T} \!+\! \frac{\left( 4 \eta^2 L^2 K \!+\! 8\eta L\right)  \left( 8 \sigma_l^2 \!+\! 32 K \sigma_g^2 \right) }{2} \nonumber\\
		&\!\!\!+  \frac{\left( 256 \eta^3 L^3 K \!+ \!128 \eta^4 L^4 K^2\right) \! \left( 8 K \sigma_l^2 \!+\! 32 K^2 \sigma_g^2 \!+\! 32 K^2 G^2 \right)  }{m(1 - \lambda)}  \nonumber\\
		&\!\!\! +\! \frac{4  \!\left( \! K L \!+\! \sqrt{K} \right) \!G^2}{ (m-1)KT} \! \sum_{t=1}^{T} \! \sum_{j=1}^{m}\! \sum_{ i\neq j}^{m}\! (  \left(  (m \!-\! 1) a_{i,j} q_{i,j,t}\!-\! 1\right)  \beta_{j,t} \!+\! 1 ) \nonumber\\
		\leq & \frac{ 256 L \left( F\left( \bar{\v{w}}_{1} \right) - F\left( \v{w}^{*}\right) \right) }{ \sqrt{mT}}+  \frac{\sqrt{m} \left( \sigma_l^2 +4 K \sigma_g^2 \right) }{2 K T^{\frac{1}{2}}} \nonumber\\
		& \!\!\!+ \frac{m\left(  \sigma_l^2 + 4 K \sigma_g^2 \right) }{256 K T}  + \frac{ \sqrt{m}\left( \sigma_l^2 + 4 K \sigma_g^2 + 4 K G^2 \right) }{128( 1 - \lambda)K T^{\frac{3}{2}}} \nonumber\\
		&\!\!\! +  \frac{m\left(  \sigma_l^2 + 4 K \sigma_g^2 + 4 K G^2 \right) }{128^2( 1 - \lambda)KT^2}\nonumber\\
		& \!\!\! +\! \frac{4  \!\left( \! K L \!+\! \sqrt{K} \right) \!G^2}{ (m-1)KT} \! \sum_{t=1}^{T} \! \sum_{j=1}^{m}\! \sum_{ i\neq j}^{m}\! (  \left(  (m \!-\! 1) a_{i,j} q_{i,j,t}\!-\! 1\right)  \beta_{j,t} \!+\! 1 ). \nonumber\\
	\end{align}
This completes the proof.

\subsection{Proof of Theorem \ref{thm_convergence}}

To facilitate analysis, we first define the global value function $V^{\pi}_t(\v{s}_t)$ of state $\v{s}_t$ at time slot $t$ under policy $\pi$. Then, the global value functions under the optimal centralized policy $\pi^*$ and the localized policy $\pi_{\textup{loc}}$ are respectively given by
\begin{equation}
	\begin{aligned}
		V^{\pi^*}_t(\v{s}_t) = \b{E}_{\v{p}_t \sim \pi^*} \left[  Q_{t}(\v{s}_t,\v{p}_t) \right] ,
	\end{aligned}
\end{equation}
and
\begin{equation}
	\begin{aligned}
		V^{\pi_{\textup{loc}}}_t(\v{s}_t) = \sum_{i=1}^{m}   \b{E}_{\v{p}_{\hat{\ca{N}}_i,t} \sim \pi_{\textup{loc}}} \left[  Q_{i,t}^{R+1}(\v{s}_{\hat{\ca{N}}_i,t}, \v{p}_{\hat{\ca{N}}_i,t})  \right]      .
	\end{aligned}
\end{equation}
Then, we have
\begin{equation}
	\begin{aligned}
		J\left(\pi_{\textup{loc}}\right) - J\left(\pi^*\right)  \leq &  \sum_{t=1}^{T} \sup_{\v{s}_t}  \left[ V^{\pi_{\textup{loc}}}_t(\v{s}_t)  - V^{\pi^*}_t(\v{s}_t) \right] \\
		\leq &  \sum_{t=1}^{T}  \left\|  V^{\pi_{\textup{loc}}}_t  - V^{\pi^*}_t \right\|_{\infty}  .
	\end{aligned}
\end{equation}
Furthermore, we define the set of localized Q-value functions generated by Algorithm \ref{D_algorithm} as $\v{Q}_{t}^{r} = \left\lbrace Q_{i,t}^{r} \right\rbrace $, and we introduce the interaction matrix $Z^{\v{Q}_{t}^{r}}\in \b{R}^{m \times m}$ of $\v{Q}_{t}^{r}$, with its $(i,j)$-th entry being derived as follows:
\begin{equation}
	\begin{aligned}
		Z^{\v{Q}_{t}^{r}}_{ij} = \max_{s_j,p_j, s_j^\prime, p_j^\prime} \max_{\v{s}_{\hat{\ca{N}}_i\backslash i}, \v{p}_{\hat{\ca{N}}_i\backslash i}} \bigg| Q_{i,t}^{r}\left(s_j \cup \v{s}_{\hat{\ca{N}}_i\backslash i}, p_j \cup \v{p}_{\hat{\ca{N}}_i\backslash i}\right)~~~&  \\
		 - Q_{i,t}^{r}\left(s_j^\prime \cup \v{s}_{\hat{\ca{N}}_i\backslash i}, p_j^\prime \cup \v{p}_{\hat{\ca{N}}_i\backslash i}\right)   \bigg|.&  \\
	\end{aligned}
\end{equation}
Then, according to the definition of the localized cost function in \eqref{thelocalizedcostfunction}, we have that each entry $Z^{\v{Q}_{t}^{r}}_{ij}$ of the interaction matrix $Z^{\v{Q}_{t}^{r}}$ is upper bounded by $ \frac{4 \left(  K L + \sqrt{K} \right)  G^2}{K} $.
Further, according to Lemma 7 in \cite{zhang2023global}, we have
	\begin{align}\label{proof_2_2}
		&J\left(\pi_{\textup{loc}}\right) - J\left(\pi^*\right) \nonumber\\
		\leq & \sum_{t=1}^{T}  \left\|  V^{\pi_{\textup{loc}}}_t  - V^{\pi^*}_t \right\|_{\infty} \nonumber\\
		\leq & \sum_{t=1}^{T} \sum_{i=1}^{m} \sum_{j=1}^{m} \sup_{\v{s}} \textup{TV} \left( \pi_{i,t}^{R+1}\left(  p_{i} \big|\v{s}_{\hat{\ca{N}}_i}\right) , \pi_{i,t}^{*}\left(  p_{i} |\v{s}\right) \right) Z^{\v{Q}_{t}^{R+1}}_{ij} \nonumber\\
		\leq & \frac{4 m  \left(  K L + \sqrt{K} \right)  G^2}{K} \nonumber\\
		& \times \sum_{t=1}^{T}  \sum_{i=1}^{m}  \sup_{\v{s}} \textup{TV} \left( \pi_{i,t}^{R+1}\left(  p_{i} \big|\v{s}_{\hat{\ca{N}}_i}\right) , \pi_{i,t}^{*}\left(  p_{i} |\v{s}\right)  \right) ,
	\end{align}
where \(\textup{TV}(\mu,\nu) = \sup_{A }|\mu(A) - \nu(A)|  \) denotes the total variation distance between two probability measures $\mu$ and $\nu$.

To proceed, we further introduce the following lemma to bound the total variation distance between the localized policy $\pi_{i,t}^{r}$ and the optimal policy $\pi_{i,t}^{*}$.
\begin{lem}\label{lemma1} (\cite[Lemma 17]{zhang2023global})
	For two arbitrary distributions $d_1,d_2$ on $\left\lbrace 1,2,\ldots,z\right\rbrace $, suppose for all $i,j \in [z]$,
	\begin{equation}
		\begin{aligned}
			\left| \log\frac{d_1(i)}{d_1(j)} - \log\frac{d_2(j)}{d_2(j)} \right| \leq \epsilon,
		\end{aligned}
	\end{equation}
	\begin{equation}
		\begin{aligned}
			\max\left\lbrace \left| \log\frac{d_1(i)}{d_1(j)} \right| ,  \left| \log\frac{d_2(i)}{d_2(j)} \right|\right\rbrace  \leq \epsilon^\prime.
		\end{aligned}
	\end{equation}
	Then, we have
	\begin{equation}
		\begin{aligned}
			\textup{TV}(d_1,d_2) \leq \frac{1}{2}z^2 e^{\epsilon^\prime}(e^{\epsilon} - 1).
		\end{aligned}
	\end{equation}
\end{lem}

According to the policy improvement rule \eqref{pi_improve} in Algorithm \ref{D_algorithm}, we have
	\begin{align}\label{proof_2_7}
		& \log \left( \frac{{\pi}_{i,t}^{r+1}\left(  p_{i} \big|\v{s}_{\hat{\ca{N}}_i}\right) }{{\pi}_{i,t}^{r+1}\left(  p_{i}^\prime \big|\v{s}_{\hat{\ca{N}}_i}\right) }  \right) \nonumber\\
		= &\gamma \b{E}_{\! p_j \sim \pi_{j,t}^{r}\!\left(\!\!\left[ \v{s}_{\hat{\ca{N}}_i}\right]_{\hat{\ca{N}}_j}\!\!\right), j \in \hat{\ca{N}}_i\! \backslash \! i}\!\biggl[\! Q_{i,t}^{r+1}(\!\v{s}_{\hat{\ca{N}}_i}, \v{p}^\prime_{\hat{\ca{N}}_i}) \! -\! Q_{i,t}^{r+1}(\!\v{s}_{\hat{\ca{N}}_i}, \v{p}_{\hat{\ca{N}}_i}) \!\biggr] \nonumber\\
		\aleq &4  \gamma (L + 1) G^2,
	\end{align}
where $ \v{p}_{\hat{\ca{N}}_i}  = p_i \cup \left\lbrace p_j \right\rbrace_{j \in \hat{\ca{N}}_i \backslash  i} $, $ \v{p}^\prime_{\hat{\ca{N}}_i}  = p_i^\prime \cup \left\lbrace p_j \right\rbrace_{j \in \hat{\ca{N}}_i \backslash  i} $, and (a) holds since $Q_{i,t}^{r}$ is upper bounded by $4 (L + 1) G^2$.
Similarly, for two different localized policies $\left\lbrace {\pi}^{r+1}_{i,t,1} \right\rbrace_{i=1}^m $ and $\left\lbrace {\pi}^{r+1}_{i,t,2} \right\rbrace_{i=1}^m$, we have
\begin{equation}
	\begin{aligned}\label{proof_2_8}
		& \left| \log \left( \frac{{\pi}^{r+1}_{i,t,1}(p_i\big| \v{s}_{\hat{\ca{N}}_i})}{{\pi}^{r+1}_{i,t,1}(p_i^\prime\big| \v{s}_{\hat{\ca{N}}_i})}  \right)  - \log \left( \frac{{\pi}^{r+1}_{i,t,2}(p_i\big| \v{s}_{\hat{\ca{N}}_i})}{{\pi}^{r+1}_{i,t,2}(p_i^\prime\big| \v{s}_{\hat{\ca{N}}_i})}  \right) \right| \\
		= & \Biggl| \hspace{-1pt}\gamma \b{E}_{\! p_j \sim \pi_{j,t,1}^{r}\!\left(\!\!\left[ \v{s}_{\hat{\ca{N}}_i}\right]_{\hat{\ca{N}}_j}\!\!\right)\!, j \in \hat{\ca{N}}_i\! \backslash \! i}\!\biggl[\! Q_{i,t}^{r+1}(\!\v{s}_{\hat{\ca{N}}_i}, \v{p}^\prime_{\hat{\ca{N}}_i}\hspace{-1pt}) \! -\! Q_{i,t}^{r+1}(\!\v{s}_{\hat{\ca{N}}_i}, \v{p}_{\hat{\ca{N}}_i}\hspace{-1pt}) \!\biggr] \\
		- &\hspace{-1pt}\gamma \b{E}_{\! p_j \sim \pi_{j,t,2}^{r}\!\left(\!\!\left[ \v{s}_{\hat{\ca{N}}_i}\right]_{\hat{\ca{N}}_j}\!\!\right)\!, j \in \hat{\ca{N}}_i\! \backslash \! i}\!\biggl[\! Q_{i,t}^{r+1}(\!\v{s}_{\hat{\ca{N}}_i}, \v{p}^\prime_{\hat{\ca{N}}_i}\hspace{-1pt}) \! -\! Q_{i,t}^{r+1}(\!\v{s}_{\hat{\ca{N}}_i}, \v{p}_{\hat{\ca{N}}_i}\hspace{-1pt}) \!\biggr]  \! \Biggr|\\
		\aleq &2  \gamma \sum_{j=1}^{m} \textup{TV}\left( {\pi}^{r}_{j,t,1}(\cdot| \v{s}_{\hat{\ca{N}}_j}) , {\pi}^{r}_{j,t,2}(\cdot| \v{s}_{\hat{\ca{N}}_j}) \right) Z^{\v{Q}_{t}^{r+1}}_{ij} \\
		\leq &16  \gamma (L+1) G^2 \sum_{j=1}^{m} \sup_{\v{s}_{\hat{\ca{N}}_j}} \textup{TV}\left( {\pi}^{r}_{j,t,1}(\cdot| \v{s}_{\hat{\ca{N}}_j}) , {\pi}^{r}_{j,t,2}(\cdot| \v{s}_{\hat{\ca{N}}_j}) \right),
	\end{aligned}
\end{equation}
where (a) follows from Lemma 5 in \cite{qu2020scalable} and we omit it for brevity.

Then, applying Lemma \ref{lemma1} and \eqref{proof_2_7}, \eqref{proof_2_8} and taking $\gamma \leq \frac{1}{32 m (L+1) G^2   | \ca{P} |^2} $, we can obtain that
	\begin{align}\label{proof_2_9}
		& \sum_{i=1}^{m} \sup_{\v{s}} \textup{TV} \left( \pi_{i,t}^{r+1}\left(  p_{i} \big|\v{s}_{\hat{\ca{N}}_i}\right) , \pi_{i,t}^{*}\left(  p_{i} |\v{s}\right)  \right) \nonumber \\
		\leq & \frac{1}{2} m | \ca{P} |^2 e^{4  \gamma (L+1) G^2} \nonumber  \\
		& \times \! \left(e^{16  \gamma (L+1) G^2 \sum_{i=1}^{m} \sup_{\v{s}} \textup{TV} \left( \pi_{i,t}^{r}\left(  p_{i} |\v{s}_{\hat{\ca{N}}_i}\right) , \pi_{i,t}^{*}\left(  p_{i} |\v{s}\right)  \right) } - 1 \right) \nonumber \\
		\aleq &16  \gamma m  (L+1) G^2  | \ca{P} |^2 \left( 8  \gamma (L+1) G^2 + 1  \right) \nonumber \\
		& \times   \sum_{i=1}^{m}   \sup_{\v{s}} \textup{TV} \left( \pi_{i,t}^{r}\left(  p_{i} \big|\v{s}_{\hat{\ca{N}}_i}\right) , \pi_{i,t}^{*}\left(  p_{i} |\v{s}\right)  \right) , 
	\end{align}
where (a) holds due to the fact that $e^x \leq 1 + 2x$ for $x\leq \frac{1}{2}$. Applying \eqref{proof_2_9} for $R$ times, we have
\begin{equation}
	\begin{aligned}\label{proof_2_10}
		& \sum_{i=1}^{m} \sup_{\v{s}} \textup{TV} \left( \pi_{i,t}^{R+1}\left(  p_{i} \big|\v{s}_{\hat{\ca{N}}_i}\right) , \pi_{i,t}^{*}\left(  p_{i} |\v{s}\right)  \right) \\
		\aleq & \left( 16\gamma m (L+1) G^2   | \ca{P} |^2 \left( 8 \gamma (L+1) G^2 + 1  \right) \right)^R  \\
		& \times   \sum_{i=1}^{m}   \sup_{\v{s}} \textup{TV} \left( \pi_{i,t}^{1}\left(  p_{i} \big|\v{s}_{\hat{\ca{N}}_i}\right) , \pi_{i,t}^{*}\left(  p_{i} |\v{s}\right)  \right) .\\
	\end{aligned}
\end{equation}
Plugging \eqref{proof_2_10} into \eqref{proof_2_2} and taking $\gamma \leq \frac{1}{32 m (L+1) G^2   | \ca{P} |^2} $, we have
\begin{equation}
	\begin{aligned}\label{proof_2_11}
		&J\left(\pi_{\textup{loc}}\right) - J\left(\pi^*\right) \\
		\leq & \frac{4 m \left(  K L + \sqrt{K} \right)  G^2 \left( 32 \gamma m (L+1) G^2   | \ca{P} |^2 \right)^R}{K }  \\
		& \times \sum_{t=1}^{T}  \sum_{i=1}^{m} \sup_{\v{s}} \textup{TV} \left( \pi_{i,t}^{1}\left(  p_{i} \big|\v{s}_{\hat{\ca{N}}_i}\right) , \pi_{i,t}^{*}\left(  p_{i} |\v{s}\right)  \right) .
	\end{aligned}
\end{equation}
This completes the proof.

\endgroup

\bibliography{IEEEabrv,PA_DMTL}
\bibliographystyle{ieeetr}

\end{document}